\documentclass{article} 
\usepackage[a4paper, margin=1in]{geometry}

\usepackage{hyperref}
\usepackage{url}

\usepackage[T1]{fontenc}    
\usepackage{url}            
\usepackage{booktabs}       
\usepackage{amsfonts}       
\usepackage{nicefrac}       
\usepackage{microtype}      
\usepackage{xcolor}         

\usepackage{graphicx}
\usepackage{comment}
\usepackage{amsmath}
\usepackage{amsthm}
\usepackage{amssymb}
\usepackage{bm}

\usepackage{float}
\usepackage{subfig}
\usepackage{algorithm}
\usepackage{algpseudocode}
\usepackage{tikz}
\usepackage{comment}
\usepackage{multirow}
\usepackage{comment}
\usepackage{cleveref}
\usepackage{algorithm}
\usepackage{algpseudocode}
\usepackage{amsmath,amssymb}
\usepackage{enumitem}
\usepackage{comment} 
\usepackage{hyperref}
\usepackage{bm}

\usepackage{thmtools}
\usepackage{thm-restate}

\usepackage{amssymb}
\usepackage{amsmath}
\usepackage{amsfonts}
\usepackage{amsthm}

\usepackage[utf8]{inputenc} 
\usepackage[T1]{fontenc}    
\usepackage{url}            
\usepackage{booktabs}       
\usepackage{amsfonts}       
\usepackage{nicefrac}       
\usepackage{microtype}      
  \usepackage{mathtools}

\usepackage{tikz}
\usepackage{pgfplots}
\usetikzlibrary{pgfplots.groupplots}

\usepackage{caption}
\usepackage[bottom,hang,flushmargin]{footmisc}

\newcommand{\tsn}[1]{{\left\vert\kern-0.25ex\left\vert\kern-0.25ex\left\vert #1 
    \right\vert\kern-0.25ex\right\vert\kern-0.25ex\right\vert}}

\newenvironment{itemize*}%
{\begin{itemize}[leftmargin=*,topsep=0pt]%
		\setlength{\itemsep}{0pt}%
		\setlength{\parskip}{0pt}}%
	{\end{itemize}}
\newenvironment{enumerate*}%
{\begin{enumerate}[leftmargin=*,topsep=0pt]%
		\setlength{\itemsep}{0pt}%
		\setlength{\parskip}{0pt}}%
	{\end{enumerate}}

\allowdisplaybreaks

\newtheorem{theorem}{Theorem}[section]

\newtheorem{lemma}[theorem]{Lemma}

\newtheorem{proposition}[theorem]{Proposition}

\newtheorem{thm}{Theorem}[section]
\newtheorem{lem}{Lemma}[section]

\newtheorem{prop}{Proposition}[section]

\newtheorem{fact}{Fact}[section]

\newtheorem{rem}{Remark}[section]

\def\approxcorrect{\cmark\kern-1.4ex\raisebox{.30ex}{$\xmark$}}

\newcommand{\idxn}[1][]{\ifthenelse{\equal{#1}{}}{\mathbb{INDQ}_n}{\mathbb{INDQ}_{#1}}}

\newcommand{\beq}{\begin{equation}}

\newcommand{\eeq}{\end{equation}}

\definecolor{emmanuel}{RGB}{255,127,0}

\def \endprf{\hfill {\vrule height6pt width6pt depth0pt}\medskip}

\author{Zihan Wang \& Arthur Jacot \\
Courant Institute of Mathematical Sciences\\
New York University\\
New York, NY 10012, USA \\
\texttt{\{zw3508,arthur.jacot\}@nyu.edu}}

\begin{document}
\title{Implicit bias of SGD in $L_{2}$-regularized linear DNNs:
One-way jumps from high to low rank}


\maketitle
\begin{abstract}
The $L_{2}$-regularized loss of Deep Linear Networks (DLNs) with
more than one hidden layers has multiple local minima, corresponding
to matrices with different ranks. In tasks such as matrix completion,
the goal is to converge to the local minimum with the smallest rank
that still fits the training data. While rank-underestimating minima
can be avoided since they do not fit the data, GD might get
stuck at rank-overestimating minima. We show that with SGD, there is always a probability to jump
from a higher rank minimum to a lower rank one, but the probability
of jumping back is zero. More precisely, we define a sequence of sets
$B_{1}\subset B_{2}\subset\cdots\subset B_{R}$ so that $B_{r}$
contains all minima of rank $r$ or less (and not more) that are absorbing
for small enough ridge parameters $\lambda$ and learning rates $\eta$:
SGD has prob. 0 of leaving $B_{r}$, and from any starting point there
is a non-zero prob. for SGD to go in $B_{r}$.
\end{abstract}

\section{Introduction}

Several types of algorithmic bias have been observed in DNNs for a
range of architectures \cite{gunasekar_2018_implicit_bias,soudry2018implicit,Moroshko_2020_implicit_bias_diag,Ongie_2020_repres_bounded_norm_shallow_ReLU_net}.
Understanding and characterizing these types of implicit bias is crucial
to understand the practical performances of Deep Neural Networks (DNNs).

We focus on Deep Linear Networks (DLNs) $A_{\theta}=W_{L}\cdots W_{1}$
for $\theta=(W_{1},\dots,W_{L})$, that are known to be biased towards
low-rank linear maps in a number of settings:
\begin{enumerate}
\item Adding $L_{2}$-regularization to the parameters of a DLN has the
effect of adding $L_{p}$-Schatten norm (the $L_{p}$ norm of the
singular values of a matrix) regularization to the learned matrix
for $p=\nicefrac{2}{L}$ where $L$ is the depth of the network \cite{dai_2021_repres_cost_DLN}.
\item When trained with the cross-entropy loss, Gradient Descent (GD) diverges towards infinity along direction that maximizes the margin w.r.t.
the parameter norm  \cite{Ji_2020_directional2}, leading to a form of implicit $L_{2}$-regularization
with the same bias towards low-rank matrices.
\item When the parameters are initialized with a small variance, the network
learns incrementally matrices of growing rank, thus converging to
a low-rank solution \cite{li2020towards,jacot-2021-DLN-Saddle}.
\end{enumerate}
This low-rank bias is particularly useful in the context of matrix
completion \cite{candes2009exact}, where the goal is to recover a
matrix from a subset of its entries under the assumption that the
full matrix is low rank. The task of finding the lowest rank matrices
fitting the observed entries is NP-hard, but convex approximations
can work well \cite{candes2009exact,candes_2010_convex_MC1}, as well
as DLNs \cite{keshavan_2010_MC_special_matrix_factorization,sun_2016_MC_shallow_DLN_L2_reg}.

In the deep case $L>2$, the $L_{p}$-Schatten norm becomes non-convex
(because $p=\nicefrac{2}{L}<1$) and there are multiple local minima
in the $L_{2}$-regularized loss, each corresponding to matrices with different ranks (similarly with cross-entropy there could
be multiple directions that locally minimize the rank). Which of these
local minima will GD converge to?

We will see how Stochastic Gradient Descent (SGD) can lead the dynamics to jump between local minima
with a bias towards low-rank minima.

\subsection{Contributions}

In this paper, we focus on the implicit bias of SGD in Deep Linear
Networks (DLNs) of depth
$L$ larger than $2$ with $L_{2}$-regularization, when trained on Matrix Completion (MC) tasks.

We first describe the many critical points of the $L_{2}$-regularized
loss, showing that for small enough ridge, all critical points that
are not local minima are avoided almost surely. We then split the
local minima into three groups, depending on whether they recover
the `true rank', underestimate, or overestimate it.

We show that the rank-underestimating minima can easily be avoided by taking a small enough
ridge $\lambda$, but no such strategy exists to avoid rank-overestimating minima
with GD.

However we show SGD has a small but non-zero
probability of jumping from any minimum to a lower rank minimum, but
the probability of jumping to a higher rank minimum is zero. More
precisely, we define sets $B_{r}$ that
contain all minima of rank $r$ or less and show that they are absorbing:
the probability for SGD to leave this set is zero, but the probability
for SGD to move from outside of this set to inside (in sufficiently
many steps) is non-zero.

This suggests that rank-overestimation can be avoided if we continue
SGD training long enough (but not too long), since the rank will decrease
incrementally. This illustrates the low-rank bias of SGD.

\subsection{Related Works}

The low-rank bias of DLNs has been observed in a number of different
settings: for example as a result of $L_{2}$-regularization or training
with the cross-entropy loss \cite{dai_2021_repres_cost_DLN}, and
as a result of small initializations \cite{arora_2018_depth_speed_Lp,arora_2019_matrix_factorization,li2020towards,jacot-2021-DLN-Saddle}.
These results rely on similar tools such as the balancedness condition,
however the underlying training dynamics leading to sparsity are very
distinct.

Motivated by the empirical observation that SGD improves generalization
\cite{init_lecun2012,Keskar_2016_generalization_SGD}, there has been
interest in the implicit bias of SGD. There is a line of work approximating
SGD with different Stochastic Differential Equations (SDEs) \cite{mandt_2017_SGD_bayesian_inference,SGD_Smith2018,jastrzkebski_2017_SGD_Langevin,chaudhari_2018_SGD_potential},
sometimes approximating the parameter dependent noise covariance with
a fixed scalar multiple of the identity thus leading to Langevin dynamics
\cite{jastrzkebski_2017_SGD_Langevin}, and in general studying the
resulting steady-state distributions \cite{chaudhari_2018_SGD_potential}.
These SDE approximations require small learning rates \cite{li_2017_SGD_SDE_approx},
but approximations to capture the effect of large learning rates have been proposed too \cite{li_2017_SGD_SDE_approx,smith_2021_SGD_bias}.

These works however focus on the bias of SGD in parameter space, showing
e.g. that it can be interpreted as changing the potential/loss \cite{chaudhari_2018_SGD_potential},
or adding a regularization term \cite{smith_2021_SGD_bias}. More
recent work has focused on the bias of SGD in diagonal linear networks
\cite{Pesme_2021_SGD_bias_diag_nets,vivien_2022_label_noise_sparsity}
leading to a sparsity effect in the vector represented by this network.

We focus on the effect of SGD in the context of deep fully-connected
linear networks with $L_{2}$-regularization, showing that SGD strengthens
the already existing low-rank bias induced by $L_{2}$-regularization.
To our knowledge our work is also unique in that it does not rely
on a SDE/continuous approximation.

\section{Setup}

We study Deep Linear Networks (DLNs) of depth $L$ and widths $w_{0}=d_{in},w_{1},\dots,w_{L}=d_{out}$
\[
A_{\theta}=W_{L}\cdots W_{1},
\]
for the $w_{\ell}\times w_{\ell-1}$ weight matrices $W_{\ell}$ and
parameters $\theta=(W_{1},\dots,W_{L})$. We will always assume that
the number of neurons in the hidden layers is sufficiently large $w_{\ell}\geq\min\{d_{in},d_{out}\}$
so that any $d_{out}\times d_{in}$ matrix $A$ can be recovered for
some parameters $\theta$: $A=A_{\theta}$.

\subsection{Matrix Completion}

We consider the $L_{2}$-regularized loss
\[
\mathcal{L}_{\lambda}(\theta)=C(A_{\theta})+\lambda\left\Vert \theta\right\Vert ^{2}
\]
where $C$ is a loss on matrices such as the Matrix Completion (MC)
loss
\[
C(A)=\frac{1}{2N}\sum_{(i,j)\in I}\left(A_{ij}^{*}-A_{\theta,ij}\right)^{2}
\]
where $A^{*}$ is the true matrix we want to recover and $I\subset\{1,\dots,d_{out}\}\times\{1,\dots,d_{in}\}$
is the set of observed entries of $A^{*}$ of size $N=\left|I\right|$.
While it is not possible in general to recover an entire matrix $A^{*}$
from a subset of its entries, it is possible if $A^{*}$ is assumed
to be low rank.

The ideal goal is to find the matrix $\hat{A}$ with lowest rank that
matches the observed entries. We define the smallest rank as the smallest integer $r^*$ such that $\inf_{A : \mathrm{Rank} A\leq r^*} C(A)=0$. Note that one could also define $r^*$ to be the smallest integer where this infimum is a attained at a finite matrix $A$, which can be higher in MC problems where filling in infinitely large entries can allow for lower ranks fitting functions. In the main we restrict ourselves to the first definition, but we discuss the second choice and its implications in the appendix.

Since finding the minimal rank solution is NP-hard in general \cite{candes2009exact}, a popular approximation is to find the matrix $\hat{A}$ that minimizes
the MC loss with a nuclear norm regularization
\[
\min_{A}\frac{1}{2N}\sum_{(i,j)\in I}\left(A_{ij}^{*}-A_{ij}\right)^{2}+\lambda\left\Vert A\right\Vert _{*},
\]
where the nuclear norm is the sum of the singular values of $A$:
$\left\Vert A\right\Vert _{*}=\sum_{i=1}^{\mathrm{Rank}A}s_{i}(A)$.
This loss is convex and can be efficiently minimized, and it has been
shown that it recovers the true matrix $A^{*}$ with high probability
with an almost optimal number of observations \cite{candes2009exact,candes_2010_convex_MC1}.

DLNs have also been used effectively in Matrix Completion, thanks
to their implicit low-rank bias. The importance of low-rank bias in
the Matrix Completion setting, makes it ideal to study the implicit
bias of SGD in DLNs.

\subsection{Representation Cost}

The low-rank bias of $L_{2}$-regularized DLNs can be understood in
terms of the representation cost $R(A;L)$ of DLNs, which equals the
minimal parameter norm required to represent a matrix $A$ with a
DLN of depth $L$:
\[
R(A;L)=\min_{\theta:A=A_{\theta}}\left\Vert \theta\right\Vert ^{2}.
\]

As observed in \cite{dai_2021_repres_cost_DLN}, the representation
cost of DLNs equals the $L_{p}$-Schatten norm $\left\Vert A\right\Vert _{p}^{p}$
(the $L_{p}$ norm of the singular matrices) of $A$ for $p=\nicefrac{2}{L}$:
\[
R(A;L)=L\left\Vert A\right\Vert _{\nicefrac{2}{L}}^{\nicefrac{2}{L}}:=L\sum_{i=1}^{\mathrm{Rank}A}s_{i}(A)^{\nicefrac{2}{L}}.
\]
This implies that the $L_{2}$-norm regularization in parameter space
can be interpreted as adding a $L_{p}$-Schatten norm regularization
in matrix space:
\[
\min_{\theta}C(A_{\theta})+\lambda\left\Vert \theta\right\Vert ^{2}=\min_{A}C(A)+\lambda L\left\Vert A\right\Vert _{\nicefrac{2}{L}}^{\nicefrac{2}{L}}.
\]

For shallow networks ($L=2$), the representation cost equals the
nuclear norm $R(A;2)=2\left\Vert A\right\Vert _{*}$. The loss has
only global minima and strict saddles, thus guaranteeing convergence
with probability 1 to global minimizers $\hat{\theta}$ of the LHS,
and the represented matrix $A_{\hat{\theta}}$ then minimizes the
RHS. We therefore simply recover the convex relaxation of Matrix Completion,
with the advantage that the loss $\mathcal{L}_{\lambda}(\theta)$
is differentiable everywhere, so that it can be optimized with vanilla
GD \cite{sun_2016_MC_shallow_DLN_L2_reg}.

In the deep ($L>2$) case however, the representation cost $R(A;L)=L\left\Vert A\right\Vert _{\nicefrac{2}{L}}^{\nicefrac{2}{L}}$
is non-convex, and both RHS and LHS may have distinct local minima
with varying rank. We will describe these local minima and show that
for small enough ridge $\lambda$, all other critical points are strict
saddles or strict minima. This implies that GD initialized
at a random point will almost surely converges to a local minimum
\cite{lee_2019_strict_saddle}. It only remains to understand to which
local minimum GD converges to.

There are multiple local minima with different ranks, for example
the zero parameters $\theta=0$ corresponding to the zero matrix $A_{\theta}=0$
is always a local minimum. On the other hand, there might be local
minima that overestimate the `true rank' that we want to recover.

For GD with a Gaussian initialization, there is a non-zero
probability to converge to any local minimum. On the other hand, we
will see how SGD can jump from local
minima to local minima.

\subsection{Stochastic Gradient Descent}

We consider SGD with replacement, that is at each time step $t$ an
index $(i_{t},j_{t})$ is sampled uniformly from the index set $I$,
independently from the previous iterations. The parameters are then
updated according to the learning rate $\eta$
\[
\theta_{t+1}=(1-2\eta\lambda)\theta_{t}-\frac{\eta}{2}\nabla_{\theta}\left(A_{i_{t}j_{t}}^{*}-A_{\theta_{t},i_{t}j_{t}}\right)^{2}.
\]
Note that due to the $L_{2}$-regularization there remains noise even
at the local minimizers, in contrast without $L_{2}$-regularization
there is neither noise nor drift at the global minima of the loss.
Thus with $L_2$-regularization the dynamics never completely stop,
making it possible for SGD to jump from one local minimum to another.
\begin{rem}
A number of previous works have approximated SGD by GD with Gaussian
noise, the simplest of which is to approximate SGD by Langevin dynamics.
Under this approximation, there is always a likelihood of jumping
from local minimum to any other local minimum, with a higher likelihood
of going to (and staying at) local minima with lower loss. Our theoretical
results show a completely different behavior, where SGD may have non-zero
probability of jumping from one local minimum to another, but zero
likelihood of jumping back. Furthermore the likelihood of SGD visiting
a certain local minimum will not scale with the loss of that local
minimum, but rather its rank. This further shows that the Langevin
approximation of SGD is inadequate.
\end{rem}

\section{Main Results}

We will first give a description of the loss landscape of $L_{2}$-regularized
DLNs and then state our main result, which says that SGD has a non-zero
probability of jumping from a local minimum to a lower rank minimum,
and that once in the neighborhood of a low rank minimum, the probability
of reaching a higher rank minimum is zero.

\subsection{$L_{2}$-regularized Loss Landscape}

The correspondence of the minimizers of $\mathcal{L}_{\lambda}(\theta)$
and $C_{\lambda}(A):=C(A)+\lambda L\left\Vert A\right\Vert _{\nicefrac{2}{L}}^{\nicefrac{2}{L}}$
extends to their local minima, and for small enough ridge $\lambda$,
all other critical points are strict saddles/maxima:
\begin{thm}
If $\hat{\theta}$ is a local minimum of $\mathcal{L}_{\lambda}(\theta)$,
then $A_{\hat{\theta}}$ is a local minimum of $C_{\lambda}(A)$.
Conversely, if $\hat{A}$ is a local minimum of $C_{\lambda}(A)$
then there is a local minimum $\hat{\theta}$ of $\mathcal{L}_{\lambda}(\theta)$
such that $\hat{A}=A_{\hat{\theta}}$.

Furthermore, for $\lambda$ small enough, all other critical points
$\hat{\theta}$ of $\mathcal{L}_{\lambda}(\theta)$ are strict saddles
or local maxima, in the sense that the Hessian $\mathcal{H}\mathcal{L}_\lambda(\hat{\theta})$
has a strictly negative eigenvalue.
\end{thm}

We can therefore focus on the local minima, since any other critical
point will be avoided with probability 1 \cite{lee_2019_strict_saddle}.

The critical points of the $L_{2}$-regularized loss are \emph{balanced},
i.e. $W_{\ell}^{T}W_{\ell}=W_{\ell-1}W_{\ell-1}^{T}$ for all $\ell=1,\dots,L-1$ (see Appendix).
This implies that all weight matrices have the same singular
values and the same rank $r$. We may therefore define the rank
of a critical point or minimum $\hat{\theta}$ as the rank $r$ of
any weight matrix $W_{\ell}$ which also matches the rank of the represented
matrix $A_{\hat{\theta}}$.

In general, there are several distinct local minima with different
ranks. The origin $\theta=0$ is always a local minimum, furthermore
for small enough ridge $\lambda$, there always is a local minimum
that finds the minimal rank required to fit the observed entries:
\begin{prop}
Consider a matrix completion problem with true matrix $A^{*}$ and
observed entries $I$. As $\lambda\searrow0$, there is a continuous
path of rank $r^*$ local minima $\theta(\lambda)$ of $\mathcal{L}_{\lambda}(\theta)$
such that $\lim_{\lambda\searrow0} C(A_{\theta(\lambda)})=0$.
\end{prop}

Note that finding a fitting matrix of minimal rank is known to be
a NP-hard problem in general \cite{candes2009exact}, which means
that it should in general be hard to find this local minimum. There are two types of problematic local minima:

\textbf{Rank-underestimating minima:} these are local minima such as
the origin $\theta=0$ with a rank lower than the minimal rank $r^{*}$,
so that the represented matrix $A_{\theta}$ cannot fit the observed
entries. These minima can be avoided with a small enough ridge
$\lambda$:
\begin{prop}
Given an initialization $\theta_{0}$ such that unregularized ($\lambda=0$)
gradient flow (GF) converges to a global minimum $\theta_{\infty}$
then for $\lambda$ small enough, regularized GF converges to a minimum
that with rank no smaller than $r^*$.
\end{prop}
This suggests that only the rank-overestimating minima are hard to avoid.

\textbf{Rank-overestimating minima:} these have a larger rank than $r^{*}$
and the represented matrix $A_{\theta}$ fits the observed entries
(with a small $O(\lambda)$ error). These are harder to avoid, suggesting
that the NP-hardness of finding an optimal rank $r^{*}$ fitting matrix
can be related to avoiding these minima. It might happen that there are no rank-overestimating minima, in which case GD can recover the minimal rank solution easily, but from now on we will focus on settings where these rank-overestimating minima appear and how SGD manages to avoid them.

\subsection{One-way Jumps from High to Low Rank}

We now show how SGD helps avoiding rank-overestimating local
minima. More precisely we show under conditions on the learning rate
$\eta$ and ridge $\lambda$ that there is always a (small) likelihood
of jumping from a local minimum to a local minimum of lower rank,
but the probability of jumping to a local minimum of higher rank is
zero. This suggests a strategy: train the network with a small ridge
to guarantee convergence to a minimum of at least the right rank,
and then take advantage of the SGD noise to find minima of
lower rank until finding the right rank.

For our analysis, we define a family of regions $B_{r}\subset\mathbb{R}^{P}$
of parameters $\theta$ that are:
\begin{enumerate}
\item $\epsilon_{1}$-approximately balanced: for all layers $\ell$, $\left\Vert W_{\ell}^{T}W_{\ell}-W_{\ell-1}W_{\ell-1}^{T}\right\Vert_{F}^{2}\leq\epsilon_{1}$,
\item $\epsilon_{2},\alpha$-approximately rank $r$ (or less):  for all $\ell$, $\sum_{i=1}^{\mathrm{Rank}W_{\ell}}f_\alpha(s_{i}(W_{\ell}^\top W_{\ell}))\leq r+\epsilon_2$
where $s_i(A)$ is the $i$-th singular value of $A$ and $f_\alpha(x)=
\begin{cases} 
\begin{aligned}
&\frac{1}{\alpha^2}x(2\alpha-x),\ &x\le \alpha \\
&1\ &x>\alpha
\end{aligned}
\end{cases} $.
\item $C$-bounded: $\left\Vert W_{\ell}\right\Vert _{F}^{2}\leq C$.
\end{enumerate}
Note that we chose the function $f_\alpha$ to be differentiable, and to satisfy $f_\alpha(0)=0$ and $f_\alpha(x)=1, \forall x\geq\alpha$. This yields a notion of approximate rank that converges to the true rank as $\alpha \searrow 0$. Changing $f_\theta$ to any other function with the same or similar properties should not affect the results.

Since all minima $\hat{\theta}$ of the $L_{2}$-regularized loss are balanced, the set $B_{r}$ contains all minima of rank $r$ or less for $C$ large enough and all $\epsilon_{1},\epsilon_{2}\geq0$,
and it contains no local minimum of higher rank for $\epsilon_{2}$ and $\alpha$
small enough. These sets allows us to separate local minima by rank, with a small neighborhood.

\begin{prop} For any minima $\hat{\theta}$ in $B_{r}$, we have $
\sum_{i=1}^{\mathrm{Rank}A_{\hat{\theta}}} f_\alpha(s_i(A_{\hat{\theta}})^{2/L}) \leq r + \epsilon_2
$.
\end{prop}

\begin{proof}
Since all minima $\hat{\theta}$ are balanced, $A_{\hat{\theta}}=U_L^\top S^L U_0$ where $S\in R^{d_{out}\times d_{in}}$ is the diagonal matrix of singular values of all $W_{\ell}$'s. Since for any $\ell$, $\sum_{i=1}^{\mathrm{Rank}W_{\ell}}f_\alpha(s_{i}(W_{\ell}^\top W_{\ell}))\leq r+\epsilon_2$, $A_{\hat{\theta}}$ satisfies $\sum_{i=1}^{\mathrm{Rank}A_{\hat{\theta}}} f_\alpha(s_i(A_{\hat{\theta}}^{2/L})) \leq r + \epsilon_2.$
\end{proof}
We can now state our main result, which says that the set $B_{r}$ is absorbing for all $r$, i.e. SGD starting from anywhere will always end up at some time inside $B_{r}$ and then never leave it:
\begin{thm}
For any $r\geq 0$, $\lambda, $ $C$ large enough and $\epsilon_1,\epsilon_2,\alpha, \eta$ small enough,
the set $B_{r}$ is closed
\[
\theta_t\in B_{r} \Rightarrow \theta_{t+1}\in B_{r}
\]
and for $r \geq 1$ and any parameters $\theta_{t}$ there is a time
$T=\tilde{\Omega}(\lambda^{-1}\eta^{-1})$
(i.e up to log terms) such that
\[
\mathbb{P}\left(\theta_{t+T}\in B_{r}|\theta_{t}\right)\geq\left(\frac{r}{\min \{d_{in},d_{out} \} }\right)^{T},
\]
thus for any starting point SGD will eventually reach $B_{r}$:
\[
\mathbb{P}(\exists T:\theta_{t+T}\in B_{r}|\theta_{t})=1.
\]

\end{thm}

\begin{proof}
(sketch) (1) The closedness of the set of $\epsilon_{1}$-approximately
balanced parameters follows from the fact that in the gradient flow
limit $\eta\searrow 0$, the balancedness errors $W_{\ell}^{T}W_{\ell}-W_{\ell-1}W_{\ell-1}^{T}$
decay exponentially
\[
\partial_{t}\left(W_{\ell}^{T}W_{\ell}-W_{\ell-1}W_{\ell-1}^{T}\right)=-\lambda\left(W_{\ell}^{T}W_{\ell}-W_{\ell-1}W_{\ell-1}^{T}\right).
\]
To guarantee a similar decay with SGD, we simply need to the control the $O(\eta^{2})$ terms.

Given $\epsilon_{1}$-approximately balancedness, the closedness of
the $\epsilon_{2}$-approximately rank $r$ or less parameters follows
from the fact that the dynamics resulting from the minimization of the cost $C(A_\theta)$ are very slow along the smallest singular vectors of $A_\theta$ \cite{arora_2018_DLN_convergence} but the $L_2$-regularization term pushes these small singular values towards zero. For small enough singular values, this second force dominates, thus leading to a decay towards zero.

(2) Under the event $A_{T}$ that in the steps $s$ from $t$ to $t+T-1$
all the random entries $(i_{s},j_{s})$ are sampled from the same
$r$ of the $d_{out}$ columns, one can show that the $d_{out}-r$ other columns
of $W_{L}$ decay exponentially to approximately 0, implying an approximate
rank of $r$ or less. The probability of that event is at least $\left(\frac{r}{d_{out}}\right)^{T}$.
\end{proof}

This shows the implicit bias of SGD towards low-rank matrices in matrix
completion: SGD can avoid any rank-overestimating minima given enough training steps.

Explicit bounds on $C,\alpha,\epsilon_1,\epsilon_2,\eta$ can be found in the Appendix. The bounds are rather complex, but we give here an example of acceptable rates in terms of $\lambda$: $C\sim \lambda^{-1}$,$\alpha\sim \lambda^{\frac{L+2}{L-2}}$, $\epsilon_1\sim \lambda^{\frac{L+2}{L-2}+2L-1}$, $\epsilon_2 \sim \lambda^{-1}$, and $\eta \sim \lambda^{4L-1+\frac{L+1}{L-2}}$. These rate suggest that an extremely small learning rate $\eta$ is necessary, especially for large depths $L$, thus making the likelihood of a jump appear very small. This seems in contradiction with our empirical observations that larger depths tend to make these jumps more likely. We believe our bounds could be made tighter, in particular when it comes to the dependence on the depth $L$ to better reflect our empirical observations.

We expect this result to generalize to other tasks. The first part of the theorem
(the closedness of $B_{r}$) should generalize
to costs such as the MSE loss and others, and the second part
too, under the event that one samples from the same $r$ training
points over $T$ time steps for the MSE loss, or sample from the same
$r$ classes for classification tasks.

A limitation however is that the second part of the result relies
on the fact that we sample the observed entries independently at each
time $t$ with possible replacement. In practice, the dataset is randomly
shuffled and taken in this random order, so that every observed entry
is chosen exactly once during each epoch. This would force the jumps
to happen within an epoch, which may not be possible depending on
the problem.

Another limitation is the average time required to observe one of
our predicted jumps can easily be absurdly large. To observe a jump
in reasonable time, one also needs rather large learning rates $\eta$,
leading to very noisy dynamics. This makes periodic learning rate
choices attractive, with large $\eta$ periods allowing for jumps
to lower-rank region, and low $\eta$ periods allowing for SGD to
settle around a local minimum.

Nevertheless, our result also shows that the  common approach
of approximating SGD with a SDE such as Langevin dynamics and studying
the stationary distribution (usually with full support over the parameter
space) is misleading. In contrast, our result implies that any stationary
distribution must have support inside $B_{r=1}$, thus under-estimating the true rank in general. It is thus crucial to understand
the distribution of SGD at intermediate times, when the rank has not yet collapsed to 1.

\begin{figure}
    \centering
    \begin{minipage}{.48\textwidth}
    \centering
    \vspace*{-0.6cm}
    \!\!\!\!\!\!\!\!\includegraphics[scale=0.45]{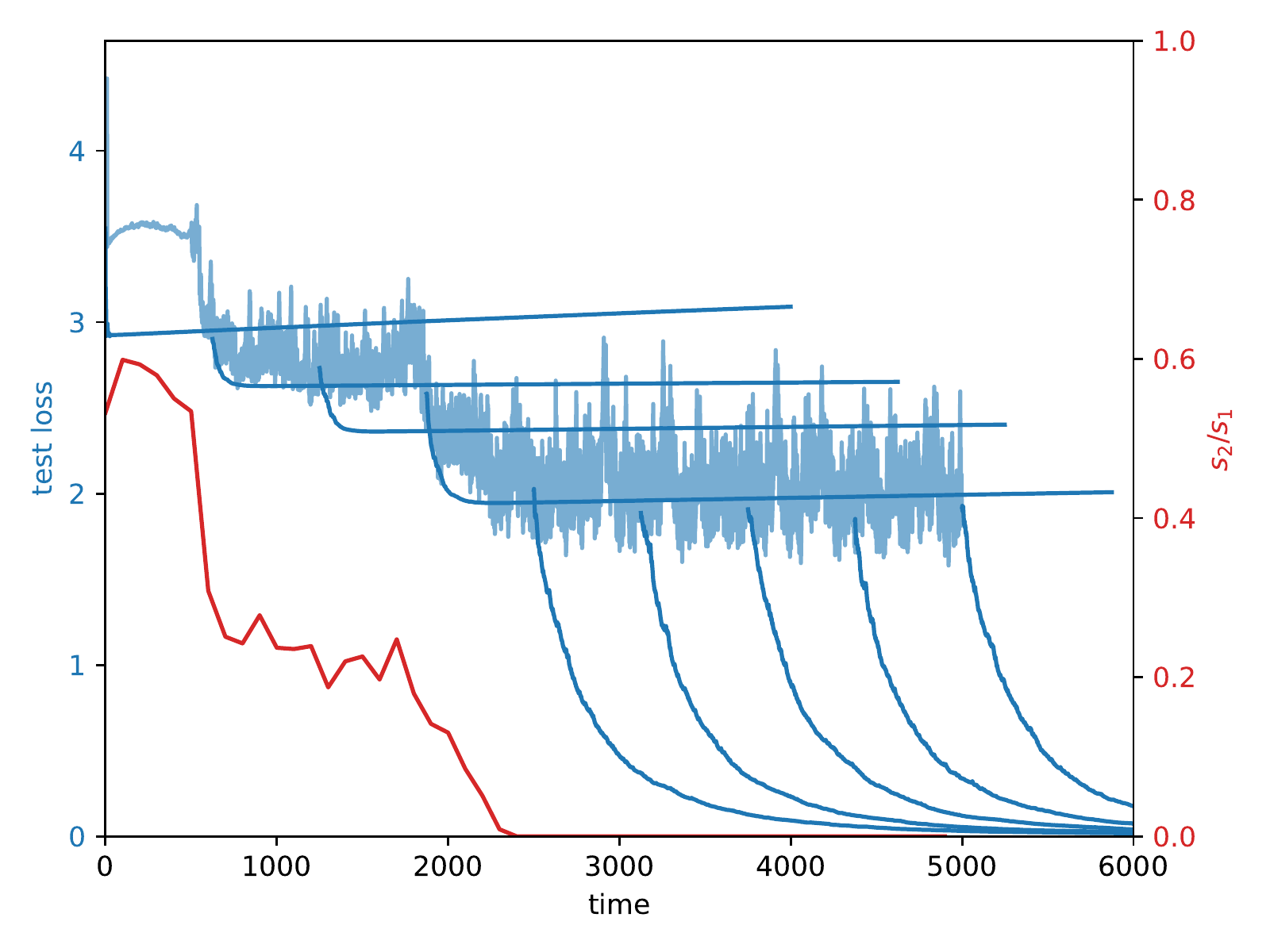}
    \caption{\textbf{Annealing Schedule:} DLN with $L=3$, $w_1=100$ on the $2\times 2$ MC problem with $\epsilon=0.25$ [Light blue] SGD with $\lambda=0.1$ and $\eta=0.2$ ($\eta=0.03$ for the first 500 steps to avoid explosion) [dark blue] at different times, we create offshoots with $\lambda=0.001$ and $\eta=0.02$ to fit the data. [red] The ratio of the second to first singular value of $A_\theta$ on the large $\lambda,\eta$ path. We see a jump around time 2000, where the output matrix becomes rank 1. The offshoots created before this jump fail to fit the missing entry while those created after succeed.}
    \label{fig:annealing}
    \end{minipage} \; \;
    \begin{minipage}{.48\textwidth}\centering
    \vspace*{-1.3cm}
    \!\!\!\!\!\includegraphics[scale=0.44]{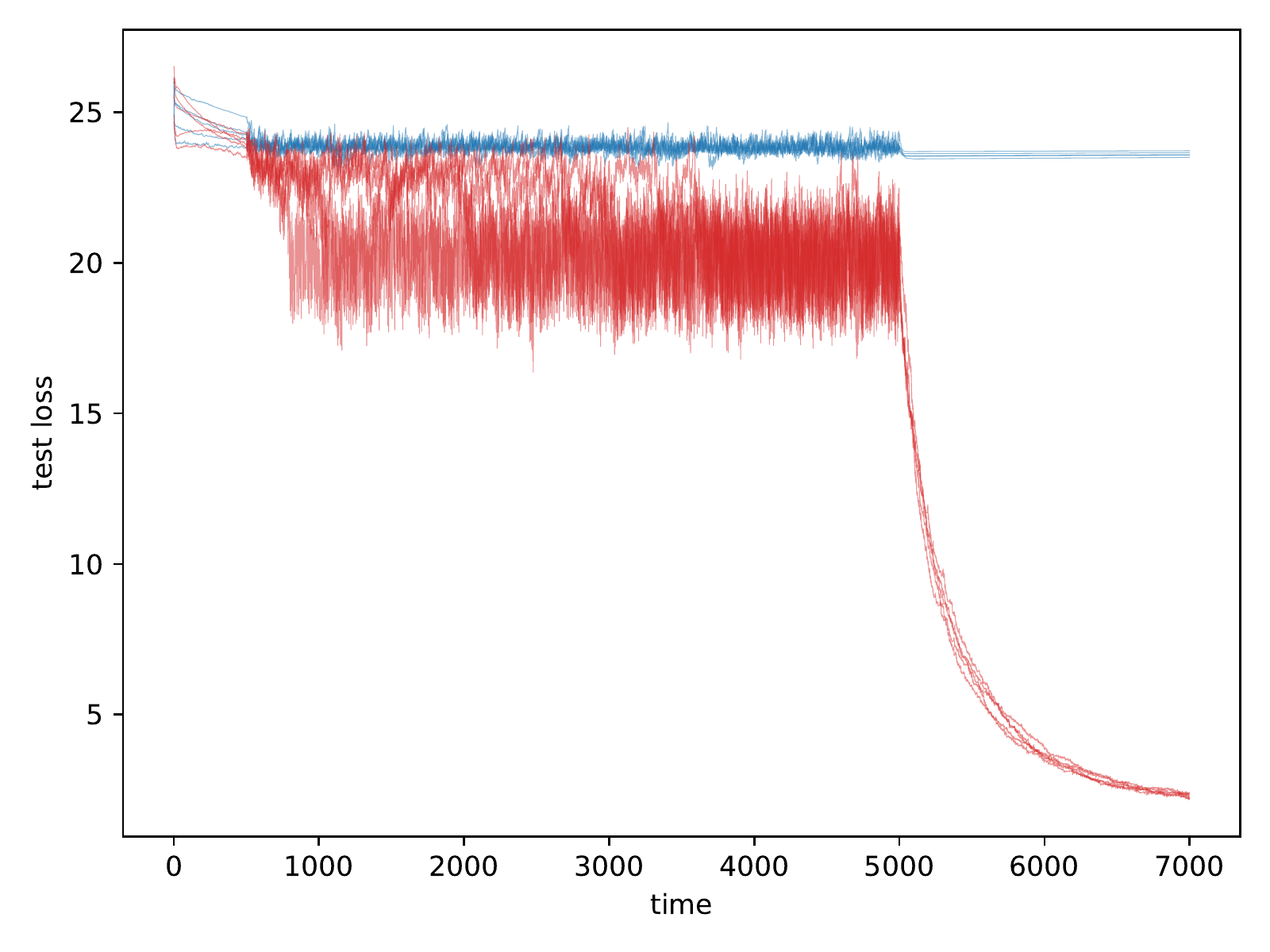}
    \caption{\textbf{Effect of depth:} We study the effect of on the MC problem with $\epsilon=0.1$. We train 5 networks of each depths $L=3$ [blue] and $L=4$ [red] with the same schedule: $\eta=0.03,\lambda=0.1$ until $t=500$, then $\eta=0.25,\lambda=0.1$ until $t=5000$ and finally $\eta=0.05,\lambda=0.001$ until the end. We see that the five depth $L=3$ networks are unable to jump in this time, while all five depth $L=4$ networks jump at different times during the first 5000 SGD steps.}
    \label{fig:depth_comparison}
\end{minipage}
\end{figure}

\subsubsection{Nonlinear networks}
Since linear networks are a simplification of nonlinear networks, it is natural to wonder whether the results presented here could be generalized to the nonlinear case. We identify two possible strategy to generalize our results:

First along the lines of \cite{chen_2023_SGD_invariant_spaces} which observes a similar phenomenon where SGD is naturally attracted to symmetric regions of the loss (where for example two neurons are identical or one neuron is dead) in nonlinear networks. The $L_2$-regularization is known to make these region more attractive \cite{jacot_2022_L2_reformulation}, which could have a compound effect with SGD. In DLNs, the regions of low rank that we prove are attractive can also be interpreted as neighborhoods of symmetric / invariant regions. 

Second, recent work has shown that $L_2$ regularized ReLU DNNs with large depths are biased towards minimizing a notion of rank over nonlinear functions, the Bottleneck rank \cite{jacot_2022_BN_rank}. We have hope that  our results could be extended to prove a similar low-rank bias with this new notion of rank. This is further motivated by the observation that such large depth networks exhibit a Bottleneck structure \cite{jacot_2023_bottleneck2} where the middle layers of the network behave approximately like linear layers.

\section{Numerical Experiments}
For our numerical experiments, we want to find a Matrix Completion problem that GD cannot solve but SGD can. In particular, we want to find a setup where GD converges with a high probability to a rank-overestimating minimum, and where SGD can jump from this minimum to a lower rank minimum in a reasonable amount of time.

It is rather difficult to find a setup that lies between the regimes where both GD and SGD work and where neither work. This is in line with previous work in the bias of SGD \cite{Pesme_2021_SGD_bias_diag_nets}: in diagonal networks a value (determined by the initialization) determines a transition between a sparse and non-sparse regimes, and SGD has the effect of pushing this value towards the sparse regime; this can have a significant sparsity effect if the original value was at the transition between regimes, but little effect if it was far into either regimes.

\begin{figure}
    \centering
    \begin{minipage}{.48\textwidth}
    \centering
    \vspace*{-1cm}
    \!\!\!\!\!\!\!\!\includegraphics[scale=0.47]{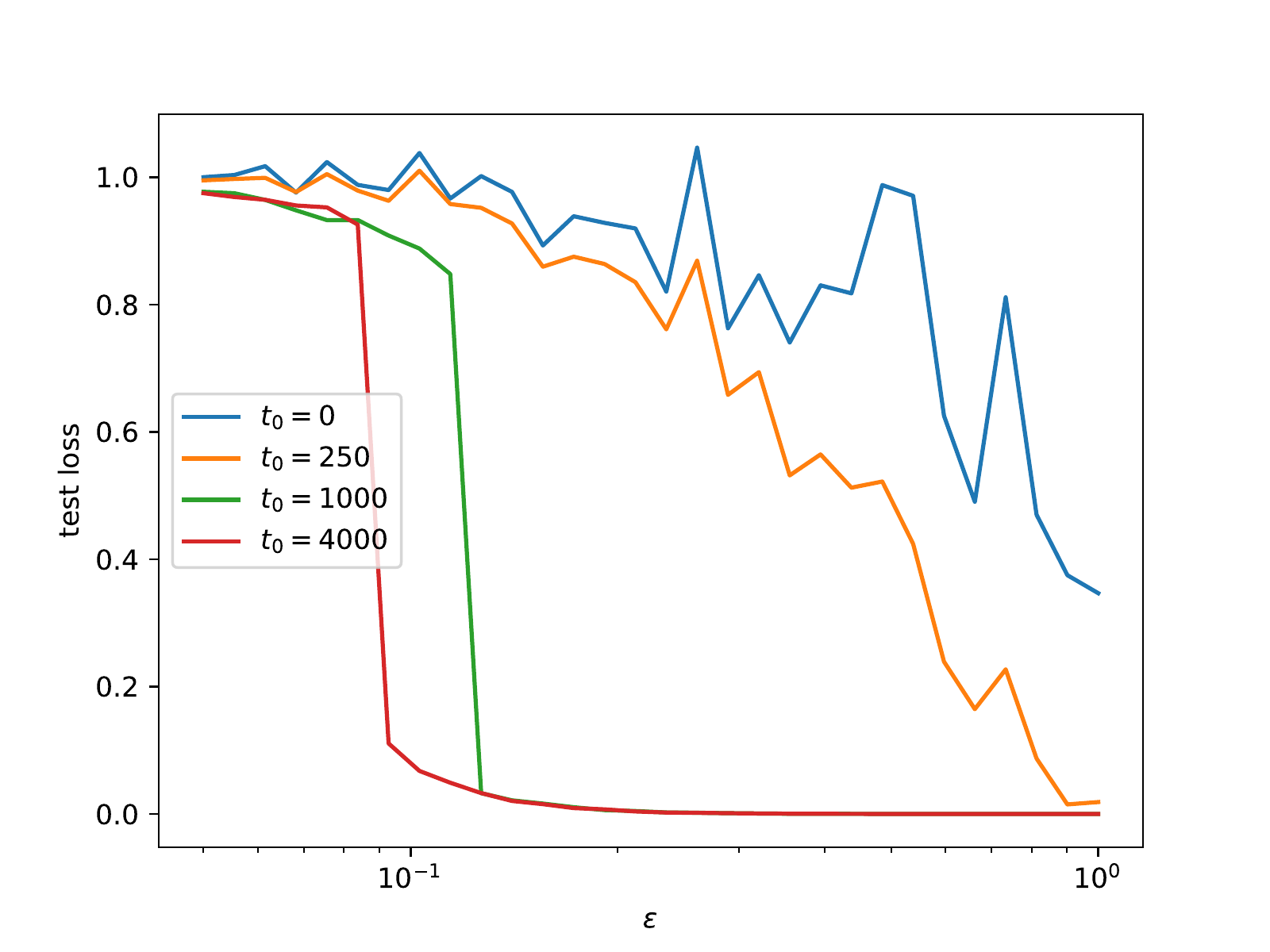}
    \caption{\textbf{Annealing accross $\epsilon$:} For a range of $\epsilon$, we train 4 networks ($L=4$,$w=100$) with an annealing schedule as in Figure \ref{fig:annealing} and plot the test loss divided by the test loss when putting zeros in the missing entries. The four networks are trained for $t_0$ steps with high noise, followed by 4000 steps in the low noise regime. Without a noise phase ($t_0=0$) the network fails to recover the rank 1 solution. Larger $t_0$ allow the network to recover it for even smaller $\epsilon$.}
    \label{fig:annealing_epsilon}
    \end{minipage} \; \;
    \begin{minipage}{.48\textwidth}\centering
    \vspace*{-0.85cm}
    \!\!\!\!\!\includegraphics[scale=0.42]{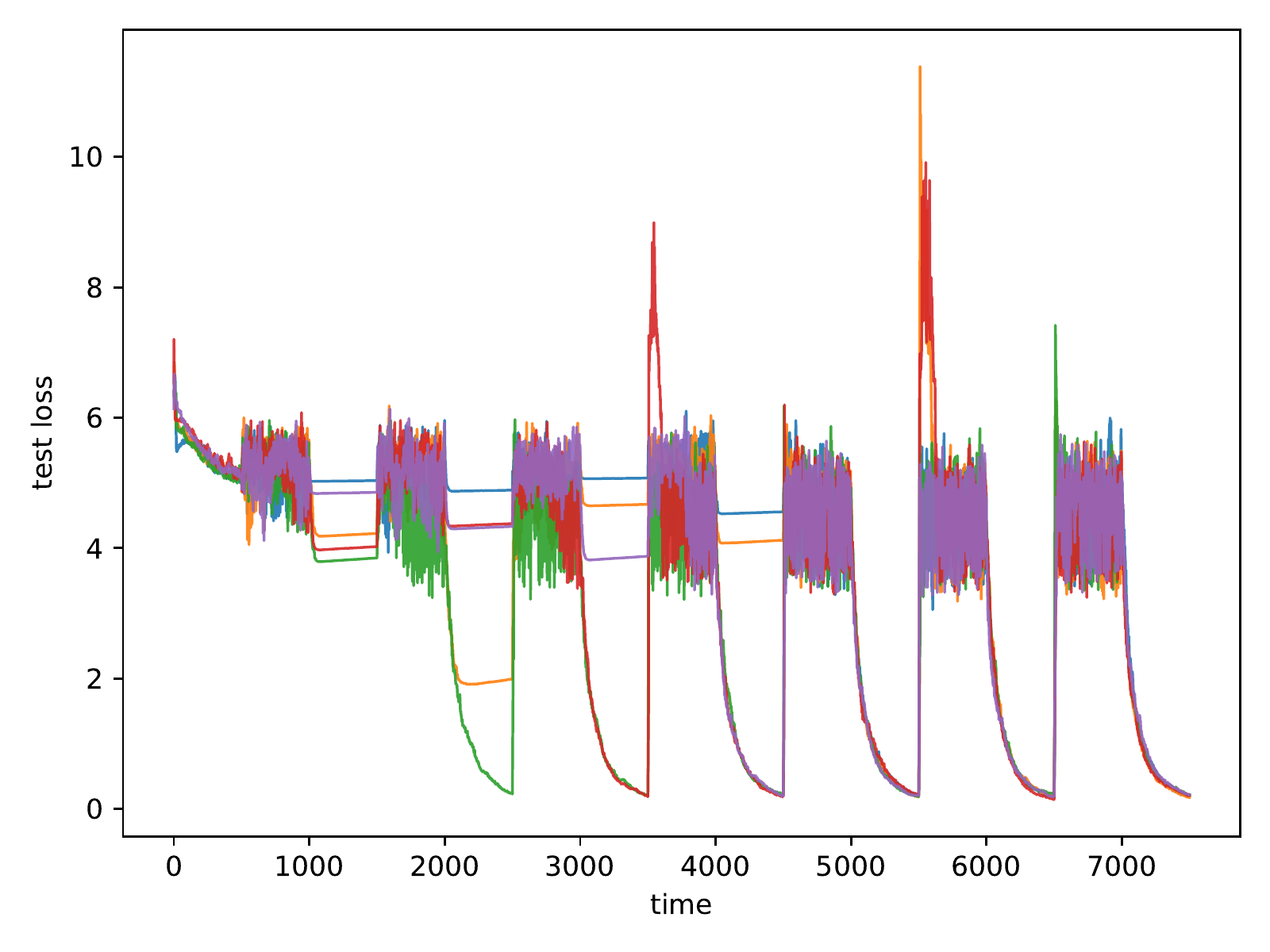}
    \caption{\textbf{Periodic Schedule:} DLN with $L=3$, $w_1=100$ on the $2\times 2$ MC problem with $\epsilon=0.2$. We plot 5 runs of SGD with periodic learning rates alternating between $\eta=0.1,\lambda=0.001$ and $\eta=0.4,\lambda=0.1$. We see that the different trials make jumps during the high $\eta,\lambda$ period. After the jump, SGD will settle at a low test error in the slow periods, allowing us to identify when the jump happened.}
    \label{fig:periodic}
\end{minipage}
\end{figure}

\begin{figure}
    \centering
    
\end{figure}

We choose a MC problem, inspired by \cite{Razim_2020}, that allows us to tune the difficulty of finding a sparse solution. We observe $3$ out of $4$ entries of a $2 \times 2$ matrix:
\[
\left(\begin{array}{cc}
1 & *\\
\epsilon & 1
\end{array}\right).
\]
Filling the missing entry $*$ with $\epsilon^{-1}$ leads to a rank $1$ matrix. The smaller $\epsilon$ is, the larger the missing entry that needs to be filled in needs to be.

In $L_2$-regularized DLNs with $L>2$ there are three local minima: the rank $0$ minimum at the origin which can easily be avoided, a set of minima that learn the rank $1$ solution, and a set of rank-overestimating minima that learn a rank $2$ solution by filling the missing entry with a small value.

For small $\epsilon$ values, GD almost always converges to a rank-overestimating minimum (see Figure \ref{fig:annealing_epsilon}). In such setup, SGD can outperform GD by jumping to a rank $1$ solution. To achieve a jump in a reasonable amount of time, we need the ridge parameter $\lambda$ and the learning rate $\eta$ to be large. But such a choice of large $\lambda,\eta$ prevent SGD from minimizing the train error.

We investigate two strategies to take advantage of both the jumping properties of large $\lambda,\eta$ and fitting properties of small $\lambda,\eta$:

\textbf{`Annealing' Schedule:} In Figure \ref{fig:annealing}, we run SGD with large $\lambda,\eta$ for some time $t_0$, waiting for a jump and then switch to small values of $\lambda,\eta$ for convergence. Another specificity is that we take a small learning rate for the first few steps, because SGD diverges if we start with a too large learning rate directly at initialization, whereas large learning rates are possible after a few steps (we do not have a theoretical explanation for that). 

We test different switching times to small $\lambda,\eta$ values, and we see clearly that if we switch after the jump at time $2000$, we obtain a rank $1$ solution, if we switch before the jump then training fails and recovers a rank 2 solution.

By changing $\epsilon$ we can tune the difficulty of finding the rank 1 solution. We see in Figure \ref{fig:annealing_epsilon} that the smaller $\epsilon$, the longer one needs to wait for a jump, and thus the longer one needs to stay in the high noise setting. We also see that without a high noise period (i.e. when we are close to GD) the network fails to recover the rank 1 solution even for $\epsilon=1$.

\textbf{Periodic Schedule:} Another strategy it to alternate between large and small $\lambda,\eta$. We see in Figure \ref{fig:periodic} how the jumps all happen during the large $\lambda,\eta$ periods. It is also interesting to see that even after SGD has settled in the vicinity of a local minimum in one of the small $\lambda,\eta$ periods, SGD can still jump to another minimum in a subsequent large $\lambda,\eta$ period.

Finally we also study the effect of depth in Figure \ref{fig:depth_comparison}, and observe that depth increases the probability of jumps. We train networks of depths $L=3$ and $L=4$ on the $2\times 2$ MC task with $\epsilon=0.1$. While for the choice $\epsilon=0.25$, a depth $L=3$ network was able to jump in a reasonable amount of time, for this smaller choice of $\epsilon$ we do not observe a jump (even with the same hyper-parameters). In contrast, the deeper networks $L=4$ all jump in a reasonable amount of time, suggesting that depth increases the likelihood of a jump.

\section{Conclusion}
We have given a description of the loss landscape of $L_2$-regularized DLNs, giving a classification of its minima by their rank. We have then shown that SGD has a non-zero probability of jumping from any higher rank minimum to a lower rank one, but it has a zero probability of jumping in the other direction. We observe these jumps empirically. To our knowledge, this is the first description of the low-rank bias of SGD in the context of fully-connected linear networks with two or more hidden layers.

Our analysis is also significantly different from previous approaches that rely on approximating SGD with a continuous stochastic process, and/or studying of the limiting distribution of this continuous process. It appears that the phenomenon of absorbing sets of different ranks cannot be recovered with a continuous approximation, and the jumps we describe happen before SGD has reached its limiting distribution. This puts into question the adequacy of the continuous approximation and limiting distribution assumption.

\bibliographystyle{plain}
\bibliography{main}
\newpage
\appendix

\title{Implicit bias of SGD in $L_{2}$-regularized linear DNNs:\\
Appendix}
\maketitle

The Appendix is organized as follows:
\begin{itemize}
    \item Section \ref{sec:Loss_landscape} contains the proofs of Theorem 2 and Proposition 3 of the main.

    \item Section \ref{sec:low-rank-bias} then describes how Theorem 5 of the main can be split into two statements.

    \item Section \ref{sec:preliminaries} state some preliminary result for the proofs.

    \item Section \ref{sec:proof_thm5.1} proves the first part of Theorem 5 from the main.

    \item Section \ref{sec:Proof_of_thm_5.2} proves the second part of Theorem 5.

    \item Section \ref{sec:low_rank_outputs} states and proves a more general version of Proposition 4 of the main.

\end{itemize}

\section{Loss Landscape} \label{sec:Loss_landscape}
\begin{prop}
    Let $\hat{\theta}$ be a critical point of the loss $\mathcal{L}_\lambda$, then $\hat{\theta}$ is balanced, i.e. $W_\ell W_\ell^T = W_{\ell+1}^T W_{\ell+1}$.
\end{prop}
\begin{proof}
    At a critical point, we have
    \[
W_{\ell+1}^{T}\cdots W_{L}^{T}\nabla C(A_{\theta})W_{1}^{T}\cdots W_{\ell-1}^{T}+2\lambda W_{\ell}=0.
\]
Thus 
\[
W_{\ell} W_\ell^T =-\frac{1}{2\lambda}  W_{\ell+1}^{T}\cdots W_{L}^{T}\nabla C(A_{\theta})W_{1}^{T}\cdots W_{\ell}^{T}=W_{\ell+1}^T W_{\ell+1}.
\]
\end{proof}

\begin{prop}
\label{prop:equivalence_local_min_strict_saddles}Let $\hat{\theta}$
be a critical point of the loss $\mathcal{L}_{\lambda}$, then:
\begin{itemize}
\item $\hat{\theta}$ is a local minimum if and only if $A_{\hat{\theta}}$
is a local minimum of $C_{\lambda}$.
\item $\hat{\theta}$ is a strict saddle/maximum if and only if $A_{\hat{\theta}}$
is a strict saddle/maximum.
\end{itemize}
\end{prop}

\begin{proof}
We know that any critical point of the $L_{2}$-regularized loss is
balanced. The parameters $\hat{\theta}=(W_{1},\dots,W_{L})$ are therefore
of the form
\[
W_{\ell}=U_{\ell}S^{\frac{1}{L}}U_{\ell-1}^{T},
\]
for some $d\times d$ diagonal $S$ (where $d=\min\{d_{in},d_{out}\}$)
and $w_{\ell}\times d$ matrices $U_{\ell}$ with orthonormal columns
($U_{\ell}^{T}U_{\ell}=I_{d}$).

(0) For any sequence of matrices $A_{1},A_{2},\dots$ converging to
$A_{\hat{\theta}}$. Given SVD decompositions $A_{i}=U_{i}S_{i}V_{i}^{T}$
(chosen so that $U_{i}$, $S_{i}$ and $V_{i}$ converge to the SVD
decomposition $A_{\hat{\theta}}=U_{L}SU_{0}^{T}$) we can construct
parameters $\theta_{i}$ with weight matrices
\begin{align*}
W_{1} & =U_{1}S_{i}^{\frac{1}{L}}V_{i}^{T}\\
W_{\ell} & =U_{\ell}S_{i}^{\frac{1}{L}}U_{\ell-1}^{T}\\
W_{L} & =U_{i}S_{i}^{\frac{1}{L}}U_{L-1}^{T}.
\end{align*}
We have (1) $A_{\theta_{i}}=A_{i}$, (2) $\theta_{i}\to\hat{\theta}$,
(3) $\left\Vert \theta_{i}\right\Vert ^{2}=L\left\Vert A_{i}\right\Vert _{\nicefrac{2}{L}}^{\nicefrac{2}{L}}$
and therefore $C_{\lambda}(A_{i})=\mathcal{L}_{\lambda}(\theta_{i})$.

(1a) If $\hat{\theta}$ is not a local minimum, there is a sequence
$\theta_{i}\to\hat{\theta}$ with $\mathcal{L}_{\lambda}(\theta_{i})<\mathcal{L}_{\lambda}(\hat{\theta})$,
thus the sequence $A_{i}=A_{\theta_{i}}$ converges to $A_{\hat{\theta}}$
and $C_{\lambda}(A_{i})\leq\mathcal{L}_{\lambda}(\theta_{i})<\mathcal{L}_{\lambda}(\hat{\theta})=C_{\lambda}(A_{\hat{\theta}})$,
implying that $A_{\hat{\theta}}$ is not a local minimum.

(1b) If $A_{\hat{\theta}}$ is not a local minimum, there is a sequence
$A_{i}\to A_{\hat{\theta}}$ with $C_{\lambda}(A_{i})<C_{\lambda}(A_{\hat{\theta}})$,
by point (0), we construct a sequence $\theta_{i}\to\hat{\theta}$
such that $\mathcal{L}_{\lambda}(\theta_{i})=C_{\lambda}(A_{i})<C_{\lambda}(A_{\hat{\theta}})=\mathcal{L}_{\lambda}(\hat{\theta})$,
proving that $\hat{\theta}$ is not a local minimum.

(2a) $\hat{\theta}$ is a strict saddle/maximum if there is a sequence
$\theta_{i}\to\hat{\theta}$ such that 
\[
\lim_{i\to\infty}\frac{\mathcal{L}_{\lambda}(\theta_{i})-\mathcal{L}_{\lambda}(\hat{\theta})}{\left\Vert \theta_{i}-\hat{\theta}\right\Vert ^{2}}<0.
\]
We then have that 
\[
\lim_{i\to\infty}\frac{C_{\lambda}(A_{\theta_{i}})-C_{\lambda}(A_{\hat{\theta}})}{\left\Vert A_{\theta_{i}}-A_{\hat{\theta}}\right\Vert _{F}^{2}}\leq\lim_{i\to\infty}\frac{\mathcal{L}_{\lambda}(\theta_{i})-\mathcal{L}_{\lambda}(\hat{\theta})}{\left\Vert \theta_{i}-\hat{\theta}\right\Vert _{F}^{2}}\frac{\left\Vert \theta_{i}-\hat{\theta}\right\Vert _{F}^{2}}{\left\Vert A_{\theta_{i}}-A_{\hat{\theta}}\right\Vert _{F}^{2}}<0,
\]
since $C_{\lambda}(A_{\theta_{i}})\leq\mathcal{L}_{\lambda}(\theta_{i})$
and $C_{\lambda}(A_{\hat{\theta}})=\mathcal{L}_{\lambda}(\hat{\theta})$
and $\lim_{i\to\infty}\frac{\left\Vert \theta_{i}-\hat{\theta}\right\Vert _{F}^{2}}{\left\Vert A_{\theta_{i}}-A_{\hat{\theta}}\right\Vert _{F}^{2}}$
is strictly positive (though possibly infinite). Therefore $A_{\hat{\theta}}$
must be a strict saddle/maximum.

(2b) $A_{\hat{\theta}}$ is a strict saddle if there is a sequence
$A_{i}\to A_{\hat{\theta}}$ with
\[
\lim_{i\to\infty}\frac{C_{\lambda}(A_{i})-C_{\lambda}(A_{\hat{\theta}})}{\left\Vert A_{i}-A_{\hat{\theta}}\right\Vert _{F}^{2}}<0.
\]
For $L>2$ we may assume that $A_{i}$ has the same rank as $A_{\hat{\theta}}$
for large enough $i$: a matrix cannot be approached with matrices
of strictly lower rank, and if it is approached with a strictly larger
rank the regularization term $\left\Vert A_{i}\right\Vert _{\nicefrac{2}{L}}^{\nicefrac{2}{L}}$
would be strictly larger.

We now construct a sequence $\theta_{i}\to\hat{\theta}$ as in (0).
Consider the map $\phi$ that maps matrices $A=USV^{T}$ in the neighborhood
of $A_{\hat{\theta}}$ with the same rank to the parameters
\begin{align*}
W_{1} & =U_{1}S^{\frac{1}{L}}V^{T}\\
W_{\ell} & =U_{\ell}S^{\frac{1}{L}}U_{\ell-1}^{T}\\
W_{L} & =US^{\frac{1}{L}}U_{L-1}^{T}.
\end{align*}
We have $\phi(A_{i})=\theta_{i}$ and $\phi(A_{\hat{\theta}})=\hat{\theta}$.
And since $\phi$ is differentiable at $A_{\hat{\theta}}$ along directions
that do not change the rank, we have
\[
\lim_{i\to\infty}\frac{\mathcal{L}_{\lambda}(\theta_{i})-\mathcal{L}_{\lambda}(\hat{\theta})}{\left\Vert \theta_{i}-\hat{\theta}\right\Vert ^{2}}=\lim_{i\to\infty}\frac{C_{\lambda}(A_{i})-C_{\lambda}(A_{\hat{\theta}})}{\left\Vert A_{i}-A_{\hat{\theta}}\right\Vert _{F}^{2}}\frac{\left\Vert A_{i}-A_{\hat{\theta}}\right\Vert _{F}^{2}}{\left\Vert \phi\left(A_{i}\right)-\phi\left(A_{\hat{\theta}}\right)\right\Vert ^{2}}<0.
\]
\end{proof}
When the ridge $\lambda$ is small enough, one can guarantee that
all critical points are either local minima or strict saddles:
\begin{prop}
For a convex cost $C(A)$ on matrix $A$ the loss $\mathcal{L}_{\lambda}(\theta)=C(A_{\theta})+\lambda\left\Vert \theta\right\Vert ^{2}$
has only local minima and strict saddles/maxima for all small enough
ridge parameter $\lambda$.
\end{prop}

\begin{proof}
Let $\theta$ be a critical point of the $L_{2}$-regularized loss
$\mathcal{L}_{\lambda}$, then $\theta$ satisfies
\[
W_{\ell+1}^{T}\cdots W_{L}^{T}\nabla C(A_{\theta})W_{1}^{T}\cdots W_{\ell-1}^{T}=-2\lambda W_{\ell}
\]
furthermore, since it is balanced, we have $W_{\ell}=U_{\ell}S^{\frac{1}{L}}U_{\ell-1}^{T}$
and therefore
\begin{equation}
U_{L}^{T}\nabla C(A_{\theta})U_{1}=-2\lambda S^{-\frac{L-2}{L}}.\label{eq:criticality_singular_values}
\end{equation}

As $\lambda\searrow0$, the critical points of the loss move continuously
(some critical points may appear or disappear but we can assume $\lambda$
to be small enough so that no such thing happen). Let us consider
a continuous path $\theta_{\lambda}$ of critical points (we will
now write $S(\lambda)$ and $U_{L}(\lambda)$ when we want to emphasize
the dependence on $\lambda$), as $\lambda\searrow0$ they converge
to parameters $\theta_{0}$. 

We now separate in two cases and show:
\begin{enumerate}
\item If the limiting matrix $A_{\theta_{0}}$ has a strictly lower rank
than $A_{\theta_{\lambda}}$ for all (sufficiently small) $\lambda>0$,
then $\theta_{\lambda}$ is a strict saddle for sufficiently small
$\lambda$.
\item If the limiting matrix $A_{\theta_{0}}$ has the same rank as $A_{\theta_{\lambda}}$
, then $\theta_{\lambda}$ is a local minimum for sufficiently small
$\lambda$.
\end{enumerate}
\textbf{(1)} For the first case, let $i$ be the index of the singular
value of $A_{\theta_{\lambda}}$ that vanishes as $\lambda\searrow0$,
i.e. $S_{ii}\searrow0$. By equation \ref{eq:criticality_singular_values},
we know that
\[
S_{ii}(\lambda)=-\left(\frac{2\lambda}{U_{L,i}^{T}(\lambda)\nabla C(A_{\theta_{\lambda}})U_{1,i}(\lambda)}\right)^{\frac{L}{L-2}}=\left(\frac{2\lambda}{\left|U_{L,i}^{T}(\lambda)\nabla C(A_{\theta_{\lambda}})U_{1,i}(\lambda)\right|}\right)^{\frac{L}{L-2}}
\]
The singular value $\left|U_{L,i}^{T}(\lambda)\nabla C(A_{\theta_{\lambda}})U_{1,i}(\lambda)\right|$
of $A_{\theta_{\lambda}}$ must converge as $\lambda\searrow0$ to
a non-zero eigenvalue, which implies that $S_{ii}(\lambda)\sim\lambda^{\frac{L}{L-2}}$.

Let us now study the Hessian of the loss at $\theta_{\lambda}$, using
the double directional derivative along any $d\theta=(dW_{1},\dots,dW_{L})$:
\begin{align*}
\mathcal{H}\mathcal{L}_{\lambda}(\theta_{\lambda})[d\theta,d\theta] & =\sum_{\ell,\ell'=1}^{L}\mathcal{H}C(A_{\theta_{\lambda}})[W_{L}\cdots dW_{\ell}\cdots W_{1},W_{L}\cdots dW_{\ell'}\cdots W_{1}]\\
 & +\sum_{\ell\neq\ell'}\mathrm{Tr}\left[\nabla C(A_{\theta_{\lambda}})W_{1}^{T}\cdots dW_{\ell}^{T}\cdots dW_{\ell'}^{T}\cdots W_{L}^{T}\right]\\
 & +2\lambda\sum_{\ell}\left\Vert dW_{\ell}\right\Vert _{F}^{2}.
\end{align*}
Taking advantage of the balancedness and evaluating along the direction
$dW_{\ell}=U_{\ell,i}(\lambda)U_{\ell-1,i}(\lambda)$, we obtain
\begin{align*}
\mathcal{H}\mathcal{L}_{\lambda}(\theta_{\lambda})[d\theta,d\theta] & =L^{2}S_{ii}(\lambda)^{2\frac{L-1}{L}}\mathcal{H}C(A_{\theta_{\lambda}})[U_{L,i}(\lambda)U_{1,i}(\lambda)^{T},U_{L,i}(\lambda)U_{1,i}(\lambda)^{T}]\\
 & -2\lambda L(L-1)+2\lambda L\\
 & \leq L^{2}\left(\frac{2\lambda}{\left|U_{L,i}^{T}(\lambda)\nabla C(A_{\theta_{\lambda}})U_{1,i}(\lambda)\right|}\right)^{2\frac{L-1}{L-2}}\sup_{A}\left\Vert \mathcal{H}C(A)\right\Vert _{op}-2\lambda L(L-2)
\end{align*}
as $L\searrow0$ the first term is of order $\lambda^{2\frac{L-1}{L-2}}$
which vanishes faster than the second term, thus guaranteeing a negative
eigenvalue of the Hessian $\mathcal{H}\mathcal{L}_{\lambda}(\theta_{\lambda})$
for all sufficiently small $\lambda$.

\textbf{(2)} Let us now consider that the rank of $A_{\theta_{\lambda}}$
does not change as $\lambda\searrow0$, which implies that all eigenvalues
$S_{ii}(\lambda)$ are either zero or lower bounded by some constant
$c$ for all sufficiently small $\lambda$.

We need to show that the Hessian has no negative eigenvalues, it is
sufficient to only check along directions $d\theta$ that keep the
network balanced since unbalanced networks have a strictly larger
parameter norm than a balanced network representing the same matrix
$A_{\theta}$, thus if there exists an escape direction that is unbalanced,
a balanced one must exist too.

Some directions that preserve balancedness are of the form $dW_{\ell}=U_{\ell,i}U_{\ell-1,i}$
for all $i$ (corresponding to changing $S$) in which case
\[
\mathcal{H}\mathcal{L}_{\lambda}(\theta_{\lambda})[d\theta,d\theta]=L^{2}S_{ii}(\lambda)^{2\frac{L-1}{L}}\mathcal{H}C(A_{\theta_{\lambda}})[U_{L,i}(\lambda)U_{1,i}(\lambda)^{T},U_{L,i}(\lambda)U_{1,i}(\lambda)^{T}]-2\lambda L(L-2)
\]
with the first term converging to a finite value and the second term
vanishing as $\lambda\searrow0$.

Other directions that preserves balancedness (corresponding to changing
$U_{\ell}$ for $\ell=1,\dots,L-1,$) are of the form $dW_{\ell}=dU_{\ell}S^{\frac{1}{L}}U_{\ell-1}$,
$dW_{\ell+1}=U_{\ell+1}S^{\frac{1}{L}}dU_{\ell}$ and $dW_{\ell'}=0$,
for any $dU_{\ell}$ such that $U_{\ell}^{T}dU_{\ell}+=-dU_{\ell}^{T}U_{\ell}$.
We have
\[
\partial_{\theta}A_{\theta}[d\theta]=W_{L}\cdots W_{\ell+1}dW_{\ell}\cdots W_{1}+W_{L}\cdots dW_{\ell+1}W_{\ell}\cdots W_{1}=0
\]
since $W_{\ell+1}dW_{\ell}=-dW_{\ell+1}W_{\ell}$. Furthermore
\[
\partial_{\theta}^{2}A_{\theta}[d\theta,d\theta]=2W_{L}\cdots dW_{\ell+1}dW_{\ell}\cdots W_{1}=2U_{L}S^{L-\ell}dU_{\ell}^{T}sU_{\ell}S^{\ell}U_{1}.
\]
Thus
\begin{align*}
\mathcal{H}\mathcal{L}_{\lambda}(\theta_{\lambda})[d\theta,d\theta] & =\mathcal{H}C(A_{\theta_{\lambda}})[\partial_{\theta}A_{\theta}[d\theta],\partial_{\theta}A_{\theta}[d\theta]]\\
 & +\mathrm{Tr}\left[\nabla C(A_{\theta_{\lambda}})\partial_{\theta}^{2}A_{\theta}[d\theta,d\theta]\right]\\
 & +2\lambda\sum_{\ell}\left\Vert dW_{\ell}\right\Vert _{F}^{2}\\
 & =0-4\lambda\mathrm{Tr}\left[S_{ii}^{\frac{2}{L}}dU_{\ell}^{T}dU_{\ell}\right]+4\lambda\mathrm{Tr}\left[S_{ii}^{\frac{2}{L}}dU_{\ell}^{T}dU_{\ell}\right]\\
 & =0
\end{align*}
Finally the directions that correspond to changing $U_{0}$ (changing
$U_{L}$ is analogous), we have $dW_{1}=U_{1}S^{\frac{1}{L}}dU_{0}$
for $U_{0}^{T}dU_{0}+=-dU_{0}^{T}U_{0}$ we get
\begin{align*}
\mathcal{H}\mathcal{L}_{\lambda}(\theta_{\lambda})[d\theta,d\theta] & =\mathcal{H}C(A_{\theta_{\lambda}})[U_{L}SdU_{0}^{T},U_{L}SdU_{0}^{T}]\\
 & +\mathrm{Tr}\left[\nabla C(A_{\theta_{\lambda}})\partial_{\theta}^{2}A_{\theta}[d\theta,d\theta]\right]\\
 & +2\lambda\sum_{\ell}\left\Vert dW_{\ell}\right\Vert _{F}^{2}
\end{align*}
and $\mathcal{H}C(A_{\theta_{\lambda}})[U_{L}SdU_{0}^{T},U_{L}SdU_{0}^{T}]\geq0$
and $2\lambda\sum_{\ell}\left\Vert dW_{\ell}\right\Vert _{F}^{2}\geq0$
while $\mathrm{Tr}\left[\nabla C(A_{\theta_{\lambda}})\partial_{\theta}^{2}A_{\theta}[d\theta,d\theta]\right]=0$
since $\partial_{\theta}^{2}A_{\theta}[d\theta,d\theta]=0$.
\end{proof}

Finally we prove the existence of a minimum with the minimal rank required to fit the data:
\begin{prop}
Consider a matrix completion problem with true matrix $A^{*}$ and
observed entries $I$. As $\lambda\searrow0$, there is a continuous
path of rank $r^*$ local minima $\theta(\lambda)$ of $\mathcal{L}_{\lambda}(\theta)$
such that $\lim_{\lambda\searrow0} C(A_{\theta(\lambda)})=0$.
\end{prop}

\begin{proof}
Let $A(\lambda)$ be a path of global minima of the cost $C_{\lambda}(A)$ restricted to the set of matrices of rank $r^*$ or less. The regularization ensures that the infinimum $\inf_{A : \mathrm{Rank} A \leq r^*} C_\lambda(A)=0$ is attained at a finite matrix $A$. The matrix $A(\lambda)$ is also a (possibly non-global) minimum of the non-restricted
loss (since when going along directions that increase the rank of
$A(\lambda)$, the regularization term increases at a rate of $d^{\nicefrac{2}{L}}$
for $d$ the distance), and there are local minima $\theta(\lambda)$
of $\mathcal{L}_{\lambda}(\theta)$ such that $A(\lambda)=A_{\theta(\lambda)}$. By the definition of the minimal rank $r^*$, we know that $\inf_{A : \mathrm{Rank} A \leq r^*} C_\lambda(A)=0$ and thus $\lim_{\lambda\searrow0} C(A_{\theta(\lambda)})=0$.
\end{proof}

\subsection{Avoiding Rank-underestimating minima}

With a small enough ridge $\lambda$ and learning rate $\eta$, one can guarantee that GD will avoid all rank-underestimating local minima:
\begin{prop}
Given an initialization $\theta_{0}$ such that unregularized ($\lambda=0$)
gradient flow (GF) converges to a global minimum $\theta_{\infty}$
then for $\lambda$ small enough, regularized GF converges to a minimum
that with rank no smaller than $r^*$.
\end{prop}
\begin{proof}
We know that $L_2$-regularized GF $\theta_\lambda(t)$ converges to unregularized GF $\theta(t)$ as $\lambda \searrow 0$. There is thus a time $t_0$ such that for all $\lambda$ small enough, \[
C_\lambda(\theta_\lambda(t_0)) < \inf_{A : \mathrm{Rank} A < r^*} C(A).
\]
Since the loss can only decrease after this time $t_0$, we know that GF will converge to a minima with smaller loss, which implies that it will converge to a solution of rank $r^*$ or more.
\end{proof}

In settings where the infimum $\inf_{A:\mathrm{Rank} A \leq r^*} C(A)$ is not attained, we can define another notion of smallest rank $\tilde{r}^* > r^*$ to be the smallest integer where this infimum is attained. One could wonder under which conditions one can avoid minima with rank $<\tilde{r}^*$. A similar result can be proven, though we require an additional assumption:

\begin{prop}
Given an initialization $\theta_{0}$ such that unregularized ($\lambda=0$)
gradient flow (GF) converges to a global minimum $\theta_{\infty}$
such that the loss is $\beta$-PL in a neighborhood of $\theta_{\infty}$,
then for $\lambda$ small enough, regularized GF converges to a minimum
rank no smaller than $\tilde{r}^*$.
\end{prop}
\begin{proof}
If we let the ridge $\lambda$ go to zero we have that GF trained
with $\lambda$-weight decay $\theta_{t,\lambda}$ converges to GF
without weight decay: $\theta_{t,\lambda}\to\theta_{t}$ as $\lambda\searrow0$.
We can therefore choose a time $T_{0}$ such that for all small enough
$\lambda$, the ball $B(\theta_{t_{0},\lambda},R_{\lambda}=\sqrt{\mathcal{L}_{\lambda}(\theta)})$
lies in the neighborhood of $\theta_{\infty}$ where the loss is $\beta$-PL.
We can apply Lemma \ref{lem:PL-inequ-L2} to obtain that at time $t_{0}+T_{\lambda}$
the loss will be $O(\lambda)$. Since the loss will only decrease
after that, we know that GF will converge to a local minimum with
$O(\lambda)$ loss.

Let us now assume by contradiction that there is a sequence $\lambda_{1}>\lambda_{2}>\dots$
with $\lambda_{n}\to0$ such that the minimum $\theta_{\infty,\lambda_{n}}$
that GF with ridge $\lambda_{n}$ converges to is rank-underestimating,
i.e $\mathrm{Rank}A_{\theta_{\infty,\lambda_{n}}}<\tilde{r}^{*}$. Since $\lambda_n\left\Vert \theta_{\infty,\lambda_{n}}\right\Vert \leq\mathcal{L}_{\lambda}(\theta_{\infty,\lambda_{n}})=O(\lambda_{n})$,
we know that the $\theta_{\infty,\lambda_{n}}$ are bounded, which
implies the existence of a convergent subsequence that converges to
parameters $\tilde{\theta}$ which by continuity of $\theta\mapsto A_{\theta}$
and $\theta,\lambda\mapsto\mathcal{L}_{\lambda}(\theta)$ satisfies
$\mathrm{Rank}A_{\tilde{\theta}}<\tilde{r}^{*}$ and $\mathcal{L}(\tilde{\theta})=0$,
which is in contradiction with the assumption that $r^{*}$ is the
smallest fitting rank. 
\end{proof}
\begin{lem}
\label{lem:PL-inequ-L2}Let the loss $\mathcal{L}$ satisfy the $\beta$-PL
inequality ($\frac{1}{2}\left\Vert \nabla\mathcal{L}(\theta)\right\Vert ^{2}\geq\beta\mathcal{L}(\theta)$)
in a ball of radius $R=\sqrt{\frac{\mathcal{L}(\theta_{0})+\lambda_{0}\left\Vert \theta\right\Vert ^{2}}{\beta}}$
around initialization $\theta_{0}$ for some $\lambda_{0}$, then
for all $\lambda\leq\lambda_{0}$ there is a time $T_{\lambda}\leq T_{0}-\frac{2}{\beta}\log\lambda$
for some $T_{0}$ such that GF $\theta_{t,\lambda}$ on the $L_{2}$-regularized
loss $\mathcal{L}_{\lambda}$ satisfies $\mathcal{L}_{\lambda}(\theta_{T_{\lambda},\lambda})=\lambda k_{0}$
for some $k_{0}$ that depends continuously on $\beta,\lambda_{0}$
and $\theta_{0}$ only.
\end{lem}

\begin{proof}
Let $T_{R,\lambda}$ be the first time gradient flow $\theta_{t,\lambda}$
leaves the ball of radius $R$, we will describe the dynamics before
$T_{R,\lambda}$ and then show that $T_{R,\lambda}$ is larger than
the time $T_{\lambda}$ we are interested in.

Inside the ball, we have

\begin{align*}
\left\Vert \nabla\mathcal{L}_{\lambda}(\theta)\right\Vert ^{2} & \geq\left(\left\Vert \nabla\mathcal{L}(\theta)\right\Vert -2\lambda\left\Vert \theta\right\Vert \right)^{2}\\
 & \geq\left(\sqrt{2\beta\mathcal{L}(\theta)}-2\lambda\left\Vert \theta\right\Vert \right)^{2}\\
 & \geq2\beta\mathcal{L}(\theta)-2\sqrt{2\beta\mathcal{L}(\theta)}2\lambda\left\Vert \theta\right\Vert \\
 & \geq2\beta\mathcal{L}_{\lambda}(\theta)-2\lambda\left(2\sqrt{2\beta\mathcal{L}(\theta_{0})}\left(\left\Vert \theta_{0}\right\Vert +R\right)+\beta\left(\left\Vert \theta_{0}\right\Vert +R\right)^{2}\right)\\
 & \geq\beta\left(2\mathcal{L}_{\lambda}(\theta)-\lambda k_{0}\right),
\end{align*}
for $k_{0}=\frac{2}{\beta}\left(2\sqrt{2\beta\mathcal{L}(\theta_{0})}\left(\left\Vert \theta_{0}\right\Vert +R\right)+\beta\left(\left\Vert \theta_{0}\right\Vert +R\right)^{2}\right)$. 

Let $T_{\lambda}$ be the first time that $\mathcal{L}_{\lambda}(\theta_{t,\lambda})=\lambda k_{0}$,
, then for all $t\leq\min\left\{ T_{\lambda},T_{R,\lambda}\right\} $
\begin{align*}
\partial_{t}\mathcal{L}_{\lambda}(\theta_{t,\lambda}) & =-\left\Vert \nabla\mathcal{L}_{\lambda}(\theta_{t,\lambda})\right\Vert ^{2}\\
 & \leq-2\beta\left(\mathcal{L}_{\lambda}(\theta_{t,\lambda})-\frac{\lambda k_{0}}{2}\right)\\
 & \leq-\beta\mathcal{L}_{\lambda}(\theta_{t,\lambda}),
\end{align*}
which implies that $\mathcal{L}_{\lambda}(\theta_{t,\lambda})\leq\mathcal{L}_{\lambda}(\theta_{0})e^{-\beta t}$
and thus that $T_{\lambda}\leq-\frac{\log\lambda}{\beta}-\frac{\log\nicefrac{k_{0}}{2}}{\beta}$
under the condition that this is smaller than $T_{R,\lambda}$.

Let us now show that $T_{R,\lambda}\geq T_{\lambda}$, by showing
that $\left\Vert \theta_{T_{\lambda},\lambda}-\theta_{0}\right\Vert <R$.

\begin{align*}
\left\Vert \theta_{T_{\lambda},\lambda}-\theta_{0}\right\Vert  & \leq\int_{0}^{T_{\lambda}}\left\Vert \nabla\mathcal{L}_{\lambda}(\theta_{t,\lambda})\right\Vert dt\\
 & =\int_{0}^{\mathcal{L}_{\lambda}(\theta_{0})-\lambda\frac{k_{0}}{2}}\left\Vert \nabla\mathcal{L}_{\lambda}(\theta_{t(\tau),\lambda})\right\Vert ^{-1}d\tau
\end{align*}
where we did a change of variable in time to $t(\tau)$ which is chosen
so that $\mathcal{L}_{\lambda}(\theta_{t(\tau),\lambda})=\mathcal{L}_{\lambda}(\theta_{0})-\tau$
which implies that $\partial_{\tau}t(\tau)=\left\Vert \nabla\mathcal{L}_{\lambda}(\theta_{t(\tau),\lambda})\right\Vert ^{-2}$
(so that $\partial_{\tau}\mathcal{L}_{\lambda}(\theta_{t(\tau),\lambda})=-\left\Vert \nabla\mathcal{L}_{\lambda}(\theta_{t(\tau),\lambda})\right\Vert ^{2-2}=-1$
as needed). We can now further bound
\begin{align*}
\int_{0}^{\mathcal{L}_{\lambda}(\theta_{0})-\lambda\frac{k_{0}}{2}}\left\Vert \nabla\mathcal{L}_{\lambda}(\theta_{t(\tau),\lambda})\right\Vert ^{-1}d\tau & \leq\int_{0}^{\mathcal{L}_{\lambda}(\theta_{0})-\lambda\frac{k_{0}}{2}}\frac{1}{\sqrt{\beta\mathcal{L}_{\lambda}(\theta_{t,\lambda})}}d\tau\\
 & =\frac{1}{\sqrt{\beta}}\int_{0}^{\mathcal{L}_{\lambda}(\theta_{0})-\lambda\frac{k_{0}}{2}}\frac{1}{\sqrt{\mathcal{L}_{\lambda}(\theta_{0})-\tau}}d\tau\\
 & =\frac{1}{\sqrt{\beta}}\left(\sqrt{\mathcal{L}_{\lambda}(\theta_{0})}-\sqrt{\lambda\frac{k_{0}}{2}}\right).\\
 & \leq\frac{\sqrt{\mathcal{L}_{\lambda}(\theta_{0})}}{\sqrt{\beta}}\\
 & \leq R.
\end{align*}
\end{proof}

While the PL inequality condition might be unexpected, it is actually satisfied at almost all global minima:

\begin{prop}
Given global minimum $\theta$ of a network with widths $w_{\ell}\geq d_{in}+d_{out}$
for all $\ell=1,\dots,L-1$, then for all $\epsilon>0$ there is a closeby
global minimum $\theta'$, i.e. $\left\Vert \theta-\theta'\right\Vert \leq\epsilon$,
such that the loss satisfies the PL inequality in a neighborhood of
$\theta'$.
\end{prop}
\begin{proof}
W.l.o.g., let us assume that $d_{in}\leq d_{out}$, then it is possible
to change the parameters infinitesimally to make $W_{L-1}\cdots W_{1}$
full rank while keeping the outputs $A_{\theta}=W_{L}\cdots W_{1}$
unchanged (by only changing $W_{L-1}\cdots W_{1}$ orthogonally to
$\mathrm{Im}W_{L}^{T}$ which is possible since $w_{L-1}\geq d_{in}+d_{out}$).

We now choose a neighborhood of $\theta'$ such that the smallest
singular value of $W_{L-1}\cdots W_{1}$ is lower bounded by some
$\lambda>0$. For any parameters $\theta$ in this neighborhood, the
loss satisfies the $\beta=\frac{2\lambda^{2}}{N}$-PL inequality:
\begin{align*}
\left\Vert \nabla\mathcal{L}(\theta)\right\Vert ^{2} & =\frac{1}{N^{2}}\sum_{\ell}\left\Vert W_{\ell+1}^{T}\cdots W_{L}^{T}\left[M\odot(A^{*}-A_{\theta})\right]W_{1}^{T}\cdots W_{\ell-1}^{T}\right\Vert ^{2}\\
 & \geq\frac{1}{N^{2}}\left\Vert \left[M\odot(A^{*}-A_{\theta})\right]W_{1}^{T}\cdots W_{L-1}^{T}\right\Vert ^{2}\\
 & \geq\frac{\lambda^{2}}{N^{2}}\left\Vert M\odot(A^{*}-A_{\theta})\right\Vert ^{2}\\
 & =\frac{2\lambda^{2}}{N}\mathcal{L}(\theta).
\end{align*}
 \end{proof}

The PL-inequality is typically satisfied in the NTK regime \cite{jacot2018neural, liu2020_NTK-PL-inequ}, but in the Saddle-to-Saddle regime it seems that GF converges to the vicinity of a minima that does not satisfy the PL inequality (minima that are balanced and low-rank typically do not satisfy it) \cite{li2020towards,jacot-2021-DLN-Saddle}, so that the PL-inequality might only be satisfied in a small neighborhood and with a small constant $\beta$. This suggests that in settings where the two notions of minimal rank $r^*$ and $\tilde{r}^*$ do not agree, the question of which minima GF converges to might be dependent on the regime of training we are in, with the NTK regime leading to a rank no less than $\tilde{r}^*$ and the Saddle-to-Saddle regime leading to a rank no less than $r^*$ at least for reasonable values of $\lambda$.

\section{Low Rank Bias} \label{sec:low-rank-bias}

In Theorem 5, there are two statements: (1) if $\theta_t\in B_{r,\varepsilon_{1},\varepsilon_{2},C}$ then $\theta_{t+1}\in B_{r,\varepsilon_{1},\varepsilon_{2},C}$ and (2) with a positive probability such that there exists a time $T$ such that $\theta_T\in B_{r,\varepsilon_{1},\varepsilon_{2},C}$. The following theorems give the formal expression of the two statements.

\begin{theorem}
\label{thm:main}
For weight $W_l$,$\ l=1,\dots,L$, $L\ge 3$, let $$B_{C,\varepsilon_1}:=\{\theta: \|W_l\|_F^2\le C,\ \|W_lW_l^\top-W_{l+1}^\top W_{l+1}\|_2\le \varepsilon_1,\ l=1,\dots ,L \},$$
where $C\ge C_1/2\lambda$. Define $F_\alpha(x)=\sum_{i=1}^d f_\alpha(x_i)$ for any $x\in\mathbb{R}^d$, where
$$
f_\alpha(x)=
\left\{ 
\begin{aligned}
&\frac{1}{\alpha^2}x(2\alpha-x),\ &x\le \alpha\\
&1,\ &x>\alpha
\end{aligned}
\right.
$$
Denote
$$B_{r,\varepsilon_2}:=\{\theta:F_\alpha\circ \sigma(W_l^\top W_l)\le r+\varepsilon_2,\ l=1,\dots,L\},$$
where $\sigma$ maps a matrix to its singular values and $\alpha\le\left(\frac{\lambda^2}{2(C_1+C^L)}\right)^\frac{1}{L-2}$. 
Then for any $\varepsilon_1,\varepsilon_2>0$ such that $\varepsilon_2<1/2$ and $\sqrt{\varepsilon_1}\le \frac{\lambda\alpha\varepsilon_2}{32nL(r+1)C^{\frac{L-1}{2}}\sqrt{2(C_1+C^L)}}$, if $\theta(t)\in B:=B_{C,\varepsilon_1}\cap B_{r,\varepsilon_2}$, then stochastic gradient descent iteration with learning rate $\eta\le \min\left\{\frac{C_1}{4(2(C_1+C^L)C^{L-1}+\lambda^2 C)},\frac{2\lambda\varepsilon_1}{4(C_1+C^L)C^{L-1}+\lambda^2\varepsilon_1},\frac{\lambda\alpha\varepsilon_2}{32n(r+1)(2(C_1+C^L)C^{L-1}+\lambda^2 C)},\frac{2(r+1)}{\lambda}\right\}$ satisfies $\theta(t+1)\in B$, where $n$ is the maximal widths and heights of weight matrices. 
\end{theorem}

\begin{theorem}
\label{thm:travel}
For any initialization $\theta_0$, denote $C_0:=\max_{1\le l\le L} \|W_l\|_F^2$, if $\eta\le\min\left\{\frac{C_1}{4(2(C_1+C_0^L)C_0^{L-1}+\lambda^2 C_0)},\frac{\lambda\varepsilon_1}{4(C_1+C^L)C^{L-1}+2\lambda^2 C}\right\}$ and $C\ge\frac{C_1}{\lambda},$ then for any time $T=T_0+T_1$ satisfying $T_0\ge \frac{\log(2C_0/\varepsilon_1)}{\eta\lambda}$ and $T_1\ge \frac{\log\left((4(n-r)C)/(\alpha\varepsilon_2)\right)}{2\eta\lambda}$, we have
     $$\mathbb{P}(\theta_{T}\in B_{r,\varepsilon_1,\varepsilon_2,C})\ge \left(\frac{r}{\min\{d_{in},d_{out}\}}\right)^{T_1}.$$
\end{theorem}

\section{Preliminaries of Proofs} \label{sec:preliminaries}
\subsection{Facts in Linear Algebra}
\begin{fact}
\label{fact:norm inequ}
$\|AB\|_*\le \|A\|_*\|B\|_*$, where $\|\cdot\|_*$ represents Frobenius norm or 2 norm.
\end{fact}


\begin{fact}
    \label{fact:trace}
    For matrices $A,B$ satisfy $AB$ is square, we have $|\mathrm{Tr}(AB)|\le \|A\|_F\|B\|_F.$
\end{fact}

\noindent Let $\sigma_1\ge\sigma_2\ge\cdots\ge \sigma_r$ be the singular values of a matrix $A\in \mathbb{R}^{m\times n}$, where $r=\min\{m,n\}$. We have following facts.
\begin{fact}
\label{fact:prod}
    $\sigma_i(AB)\le \sigma_1(A)\sigma_i(B)$ and $\sigma_i(AB)\le \sigma_i(A)\sigma_1(B)$ for any $i$.
\end{fact}




\begin{fact}
\label{fact:eigen}
    For a square matrix $A\in \mathbb{R}^{n\times n}$, let $\lambda_1\ge \cdots \ge \lambda_n$ be the eigenvalues. Then we have
    $$
    \sum_{i=1}^k |\lambda_i^p(A)|\le \sum_{i=1}^k\sigma_i^p(A)
    $$
    for $k=1,2,\dots,n$ and $p>0$.
\end{fact}

\subsection{Spectral Function}
For a function $f:\mathbb{R}^n\mapsto \mathbb{R}$ that preserves permutation, we consider the function $f\circ \lambda$, where $\lambda(A)$ represents all eigenvalues of symmetric matrix $A\in\mathbb{R}^{n\times n}$. We define $\mathrm{Diag}\mu$ be the diagonal matrix with its entries equal to $\mu$ and $\mathrm{diag} A=(A_{11},\dots,A_{nn})$. The following lemmas gives the first and second order derivatives of $f\circ\lambda$.

\begin{lemma}[Lemma 3.1 from \cite{lewis2001twice}]
\label{lemma:1-derivative}
$f$ is differentiable at point $\lambda(A)$ if and only if $f\circ \lambda$ is differentiable at $A$. Moreover, we have
$$
\nabla (f\circ\lambda)(A)=U(\mathrm{Diag}\nabla f(\lambda(A)))U^\top,
$$
where $U$ is a orthogonal matrix satisfying $A=U(\mathrm{Diag}\lambda(A)U^\top$.
\end{lemma}

\noindent For a decreasing sequence $\mu\in \mathbb{R}^n$, where
$$
\mu_1=\cdots=\mu_{k_1}>\mu_{k_1+1}=\cdots=\mu_{k_2}>\mu_{k_2+1}\cdots \mu_{k_r},
$$
denote $I_l=\{k_{l-1}+1,\dots,k_l\}$ for $l=1,\dots,r$. For a twice differentiable function $f$, we define vector $b(\mu)$ as
\begin{equation}
b_i(\mu)= \begin{cases}f_{i i}^{\prime \prime}(\mu) & \text { if }\left|I_l\right|=1, \\ f_{p p}^{\prime \prime}(\mu)-f_{p q}^{\prime \prime}(\mu) & \text { for any } p \neq q \in I_l\end{cases}
\end{equation}
and matrix $\mathcal{A}(\mu)$ as
\begin{equation*}
\mathcal{A}_{ij}(\mu)= \begin{cases}0 & \text { if } i=j, \\ b_i(\mu) & \text { if } i \neq j \text { but } i, j \in I_l, \\ \frac{f_i^{\prime}(\mu)-f_j^{\prime}(\mu)}{\mu_i-\mu_j} & \text { otherwise. }\end{cases}
\end{equation*}
\begin{lemma}[Theorem 3.3 from \cite{lewis2001twice}]
\label{lemma:2-derivative}
    $f$ is twice differentiable at point $\lambda(A)$ if and only if $f\circ\lambda$ is twice differentiable at $A$. Moreover, we have
    $$
    \nabla^2(f\circ\lambda)(A)[H]=\nabla^2f(\lambda(A))[\mathrm{diag}\Tilde{H},\mathrm{diag}\Tilde{H}]+\langle \mathcal{A}(\lambda(A),\Tilde{H}\circ\Tilde{H})\rangle,
    $$
    where $A=W\mathrm{Diag}\lambda(A)W^\top$ and $\Tilde{H}=W^\top HW$.
\end{lemma}

\section{Proof of Theorem \ref{thm:main}} \label{sec:proof_thm5.1}
For stochastic gradient descent, the parameter updates as
\begin{equation}
    \theta_{t+1}=(1-\eta\lambda)\theta_t-\frac{\eta}{2}\nabla_\theta(A^*_{i_tj_t}-A_{\theta_t,i_tj_t})^2.
\end{equation}
Then for each $l$, $l$-th layer's weight $W_l$ updates as
\begin{equation}
    W_l(t+1)=W_l(t)-\eta\left(W_{l+1}(t)^\top \cdots W_{L}(t)^\top G_{\theta_t,i_tj_t} W_{1}(t)^\top \cdots W_{l-1}(t)^\top+\lambda W_l(t)\right),
\end{equation}
where $G_{\theta,ij}$ is a matrix where the $(i,j)$-th entry is  $A_{\theta,ij}-A_{ij}^*$ and other entries are $0$. In the proofs below, we will omit the iteration $t$ for convenience (for example $W_l=W_l(t)$). We denote $T_l=W_{l+1}^\top\cdots W_{L}^\top G_{\theta,j} W_1^\top \cdots W_{l-1}^\top$. 

\subsection{Approximate Balance}
First we give a lemma that bounds $\|G_{\theta,ij}\|_F$.
\begin{lemma}
\label{lemma:G}
For any $C$, if $\|W_l(t)\|_F^2\le C$ for any $l=1,\dots ,L$, then $\|G_{\theta,ij}\|_F^2\le 2(C_1+C^L)$.
\end{lemma}
\noindent\textit{Proof.} By Fact \ref{fact:norm inequ}, we have $\|A_\theta\|_F^2\le \prod_{l=1}^L \|W_l\|_F^2\le C^L$. Then
\begin{equation}
    \begin{split}
        \|G_{\theta,ij}\|_F^2&=(A_{\theta,ij}-A_{ij}^*)^2\le 2\left(A_{\theta,ij}^2+(A_{ij}^*)^2\right)\\
        &\le 2\left(\|A_\theta\|_F^2+C_1\right)\le 2(C^L+C_1).
    \end{split}
\end{equation}\qed

\begin{proposition}
\label{prop:bound}
For any $C\ge \frac{C_1}{2\lambda}$, if $\|W_l(t)\|_F^2\le C$ for any $l=1,\dots ,L$, then stochastic gradient descent iteration with learning rate $\eta\le \frac{C_1}{4(2(C_1+C^L)C^L+\lambda^2 C)}$ satisfies $\|W_l(t+1)\|_F^2\le C$ for $l=1,\dots,L$.
\end{proposition}

\noindent\textit{Proof.} If $\|W_l|_F^2\le C$, we have
\begin{equation}
    \begin{split}
\|W_l(t+1)\|_F^2=\|W_l\|_F^2-2\eta\left(\lambda\|W_l\|_F^2+\mathrm{Tr}(W_l^\top T_l)\right)+\eta^2\left\|T_l+\lambda W_l\right\|_F^2
    \end{split}
\end{equation}
We estimate each part in the equation separately. We have
\begin{align*}
    \mathrm{Tr}(W_l^\top T_l)&=\mathrm{Tr}(W_1^\top \cdots W_L^\top G_{\theta,ij})\\
    &=\mathrm{Tr}(A_\theta^\top G_{\theta,ij})=A_{\theta,ij}(A_{\theta,ij}-A^*_{ij})\\
    &\ge -\frac{1}{4}(A^*_{ij})^2\ge -\frac{1}{4}C_1.
\end{align*}

By Lemma \ref{lemma:G}, we have
\begin{align*}
    \left\|T_l+\lambda W_l\right\|&_F^2\le 2\left(\|G_{\theta,ij}\|_F^2\prod_{k\neq l} \|W_k\|_F^2+\lambda^2\|W_l\|_F^2\right)\\
    &\le 2\left(2(C_1+C^L)C^{L-1}+\lambda^2C\right).
\end{align*}
Then we have
$$
\|W_l(t+1)\|_F^2\le (1-2\eta\lambda)\|W_l\|_F^2+\frac{1}{2}\eta C_1+2\eta^2(2(C_1+C^L)C^{L-1}+\lambda^2C).
$$
When $\eta\le \frac{C_1}{4(2(C_1+C^L)C^{L-1}+\lambda^2 C)}$ and $C\ge\frac{C_1}{2\lambda}$,
$$
\|W_l(t+1)\|_F^2\le (1-2\eta\lambda)\|W_l\|_F^2+\eta C_1\le C.
$$\qed

\begin{proposition}
\label{prop:balance}
    For any $\varepsilon,C>0$, if $\|W_l(t)W_l(t)^\top-W_{l+1}(t)^\top W_{l+1}(t)\|_2\le \varepsilon$ and $\|W_l(t)\|_F^2\le C$ for all $l$, then stochastic gradient descent iteration with learning rate $\eta\le \frac{2\lambda\varepsilon_1}{4(C_1+C^L)C^{L-1}+\lambda^2\varepsilon_1}$ satisfies $\|W_l(t+1)W_l(t+1)^\top-W_{l+1}(t+1)^\top W_{l+1}(t+1)\|_2 \leq \epsilon$ for $l=1,\dots,L$.
\end{proposition}
\noindent\textit{Proof.} We first compute the update of $W_lW_l^\top$:
\begin{equation*}
    \begin{split}
       W_l(t+1) &W_l(t+1)^\top = \left((1-\eta\lambda)W_l-\eta T_l\right)\left((1-\eta\lambda)W_l-\eta T_l\right)^\top\\ 
        &=(1-\eta\lambda)^2 W_lW_l^\top -(1-\eta\lambda)\eta (W_l T_l^\top+T_l W_l^\top)+\eta^2 T_lT_l^\top.
    \end{split}
\end{equation*}
Similarly, we have
\begin{equation*}
    \begin{split}
        W_{l+1}&(t+1)^\top W_{l+1}(t+1)\\
        &=(1-\eta\lambda)^2 W_{l+1}^\top W_{l+1}-(1-\eta\lambda)\eta (W_{l+1}\top T_{l+1}+T_{l+1}^\top W_{l+1})+\eta^2 T_{l+1}^\top T_{l+1}.
    \end{split}
\end{equation*}
Since 
$$
T_lW_l^\top=W_{l+1}^\top \cdots W_L^\top G_{\theta,ij} W_1^\top \cdots W_l^\top = W_{l+1}^\top T_{l+1}
$$
and
$$
W_lT_l^\top=\left(T_lW_l^\top\right)^\top=\left(W_{l+1}^\top T_{l+1}\right)^\top=T_{l+1}^\top W_{l+1},
$$
we have
\begin{align*}
    W_l(t+1)&W_l(t+1)^\top-W_{l+1}(t+1)^\top W_{l+1}(t+1)\\
    &=(1-\eta\lambda)^2(W_lW_l^\top-W_{l+1}^\top W_{l+1})+\eta^2(T_lT_l^\top-T_{l+1}^\top T_{l+1}).
\end{align*}
Then
\begin{equation}
    \begin{split}
        \|W_l(t+1)&W_l(t+1)^\top -W_{l+1}(t+1)^\top W_{l+1}(t+1)\|_2\\
        &\le (1-\eta\lambda)^2\left\|W_lW_l^\top-W_{l+1}^\top W_{l+1}\right\|_2+\eta^2\left\|T_lT_l^\top-T_{l+1}^\top T_{l+1}\right\|_2\\
        &\le (1-\eta\lambda)^2\varepsilon +\eta^2 \|G_{\theta,ij}\|_2^2 \left(\prod_{k\neq l}\|W_k\|_2^2+\prod_{k\neq l+1}\|W_k\|_2^2\right)\\
        &\le \varepsilon-2\eta\lambda \varepsilon +\eta^2\lambda^2\varepsilon+4\eta^2(C_1+C^l)C^{L-1}.
    \end{split}
\end{equation}
When $\eta\le \frac{2\lambda\varepsilon_1}{4(C_1+C^L)C^{L-1}+\lambda^2\varepsilon_1}$,
$$
\|W_l(t+1)W_l(t+1)^\top -W_{l+1}(t+1)^\top W_{l+1}(t+1)\|_2\le \varepsilon-2\eta\lambda\varepsilon+2\eta\lambda\varepsilon=\varepsilon.
$$\qed

\noindent With Proposition \ref{prop:bound} and Proposition \ref{prop:balance}, we have the following statement that $\theta_{t+1}$ is approximate balance and the weight of each layer is bounded.:
\begin{theorem}
\label{thm:balance}
    For any $\varepsilon>0$ and $C\le 1/2\lambda$, if $\theta(t)\in B_{C,\varepsilon}$, then the stochastic gradient descent iteration with learning rate $\eta\le \min\left\{\frac{C_1}{4(2(C_1+C^L)C^{L-1}+\lambda^2 C)},\frac{2\lambda\varepsilon_1}{4(C_1+C^L)C^{L-1}+\lambda^2\varepsilon_1}\right\}$ satisfies $\theta(t+1)\in B_{C,\varepsilon}.$
\end{theorem}
\subsection{Approximate Rank-$r$}
In this section, we prove the following theorem that the weight $W_l(t+1)$ of each layer is approximately rank-$r$.

\begin{theorem}
\label{thm:rank k}
    For any $\varepsilon_1,\varepsilon_2>0$ such that $\varepsilon_2<1/2$ and $\sqrt{\varepsilon_1}\le \frac{\lambda\alpha\varepsilon_2}{32nL(r+1)C^{\frac{L-1}{2}}\sqrt{2(C_1+C^L)}}$, if the number of layers $L\ge 3$ and $\theta(t)\in B$, then stochastic gradient descent iteration with learning rate $\eta\le \min\left\{ \frac{\lambda\alpha\varepsilon_2}{32n(r+1)(2(C_1+C^L)C^{L-1}+\lambda^2 C)},\frac{2(r+1)}{\lambda}\right\}$ satisfies $\theta(t+1)\in B_{r,\varepsilon_2}$, where $n$ is the maximal widths and heights of weight matrices. 
\end{theorem}

\noindent\textit{Proof.} We denote by $r_l$ the minima of height and width of $W_l$ and the singular value decomposition $W_l=\Tilde{U_l}^\top S_l \Tilde{V}_l$, where $\Tilde{U}_l$ and $\Tilde{V}_l$ are orthogonal matrices.
Let $f_\alpha$ and $F_\alpha$ be as defined in Theorem \ref{thm:main}. By Taylor's expansion, for any $l$ we have
\begin{equation}
    \begin{split}    &F_\alpha\circ\sigma\left(W_l(t+1)^\top W_l(t+1)\right)\\
    =&F_\alpha\circ\sigma\left(W_l^\top W_l-\eta W_l^\top T_l-\eta T_l^\top W_l-2\eta\lambda W_l^\top W_l+\eta^2(T_l+\lambda W_l)^\top(T_l+\lambda W_l)\right)\\
    =&F_\alpha\circ\sigma(W_l^\top W_l)-\left\langle \nabla(F_\alpha\circ\sigma)(W_l^\top W_l),\eta(W_l^\top T_l+T_l^\top W_l)\right\rangle\\
    &\quad\quad\quad\quad-\left\langle\nabla(F_\alpha\circ\sigma(W_l^\top W_l),2\eta\lambda W_l^\top W_l\right\rangle\\    &\quad\quad\quad\quad+\left\langle\nabla(F_\alpha\circ\sigma)(W_l^\top W_l),\eta^2 (T_l+\lambda W_l)^\top(T_l+\lambda W_l)\right\rangle\\
    &\quad\quad\quad\quad +\nabla^2(F_\alpha\circ\sigma)(W_l^\top W_l+\gamma\eta \Delta)[\eta\Delta,\eta\Delta],
    \end{split}
\end{equation}
where $\gamma\in(0,1)$ and $\Delta=- W_l^\top T_l- T_l^\top W_l-2\lambda W_l^\top W_l+\eta(T_l+\lambda W_l)^\top(T_l+\lambda W_l)$. By Lemma \ref{lemma:1-derivative}, we have 
$$
\nabla(F_\alpha\circ\sigma)(W_l^\top W_l)=\Tilde{V}_l^\top\mathrm{diag}\{f_\alpha'(s_1^2),\dots,f_\alpha'(s_{r_l}^2)\}\Tilde{V}_l,
$$
where $s_1,\dots,s_{r_l}$ are the entries of diagonal matrix $S_l$. Denote $\mathrm{diag}\{f_\alpha'(s_1^2),\dots,f_\alpha'(s_{r_l}^2)\}=f_\alpha'(S^2)$. Then 
$$
F_\alpha\circ\sigma\left(W_l(t+1)^\top W_l(t+1)\right)\le \sum_{i=1}^{r_l} f_\alpha(s_i^2)-2\eta\lambda f_\alpha'(s_i^2)s_i^2+2\eta\left|\mathrm{Tr}\left(\Tilde{V}_l f_\alpha'(S^2)\Tilde{V}_l W_l^\top T_l\right)\right|+\beta,
$$
where $\beta$ is the $O(\eta^2)$ term. 
Now we can estimate the trace term
\begin{equation}
    \begin{split}
        &\quad\left|\mathrm{Tr}\left(\Tilde{V}_l f_\alpha'(S^2)\Tilde{V}_l W_l^\top T_l\right)\right|\\
        &=\left|\mathrm{Tr}\left(W_1^\top \cdots W_{l-1}^\top \Tilde{V}_l f_\alpha'(S^2)\Tilde{V}_l W_l^\top \cdots W_L^\top G_{\theta, ij}\right)\right|\\
        &\le \left\|W_1^\top \cdots W_{l-1}^\top \Tilde{V}_l f_\alpha'(S^2)\Tilde{V}_l W_l^\top \cdots W_L^\top\right\|_F \left\|G_{\theta,ij}\right\|_F\\
        &\le \sqrt{2(C_1+C^L)}\left\|W_1^\top \cdots W_{l-1}^\top \Tilde{V}_l f_\alpha'(S^2)\Tilde{V}_l W_l^\top \cdots W_L^\top\right\|_F\\
        &= \sqrt{2(C_1+C^L)} \sqrt{\mathrm{Tr}\left((W_1^\top \cdots W_{l-1}^\top \Tilde{V}_l f_\alpha'(S^2)\Tilde{V}_l W_l^\top \cdots W_L^\top)^\top W_1^\top \cdots W_{l-1}^\top \Tilde{V}_l f_\alpha'(S^2)\Tilde{V}_l W_l^\top \cdots W_L^\top\right)}\\
        &= \sqrt{2(C_1+C^L)} \sqrt{\mathrm{Tr}\left(W_{l-1}\cdots W_1 W_1^\top \cdots W_{l-1}^\top \Tilde{V}_l f_\alpha'(S^2)\Tilde{V}_l W_l^\top \cdots W_L^\top W_L \cdots W_l \Tilde{V}_l f_\alpha'(S^2)\Tilde{V}_l\right)},
    \end{split}
\end{equation}
where the first inequality is from Fact \ref{fact:trace}.

Let $E_k=\sum_{i=1}^k (W_{k+1}^\top W_{k+1})^{i-1} (W_kW_k^\top- W_{k+1}^\top W_{k+1})(W_kW_k^\top)^{k-i}$ for $k<l$ and $E_k=\sum_{i=k}^L (W_{k-1} W_{k-1}^\top)^{L-i} (W_k^\top W_k- W_{k+1} W_{k+1}^\top)(W_k^\top W_k)^{i-k}$ for $k>l$. Then we have $$(W_k W_k^\top)^k =(W_{k+1}^\top W_{k+1})^k+E_k$$ for $k<l$ and $$(W_k^\top W_k)^{L-k+1} =(W_{k-1} W_{k-1}^\top)^{L-k+1}+E_k$$ for $k>l$. Thus, 
\begin{equation}
    \begin{split}
        \label{eq:expansion}
        &\quad \mathrm{Tr}\left(W_{l-1}\cdots W_1 W_1^\top \cdots W_{l-1}^\top \Tilde{V}_l^\top f_\alpha'(S^2)\Tilde{V}_l W_l^\top \cdots W_L^\top W_L \cdots W_l \Tilde{V}_l^\top f_\alpha'(S^2)\Tilde{V}_l\right)\\
        &\le \left|\mathrm{Tr}\left((W_l^\top W_l)^{l-1}\Tilde{V}_l^\top f_\alpha'(S^2)\Tilde{V}_l (W_l^\top W_l)^{L-l+1}\Tilde{V}_l^\top f_\alpha'(S^2)\Tilde{V}_l\right)\right|+\sum_{k\neq l}|\mathrm{Tr}(\mathcal{E}_k)|\\
        &=\left|\mathrm{Tr}\left(\Tilde{V}_l^\top (f_\alpha'(S^2))^2S^{2L}\Tilde{V}_l\right)\right|+\sum_{k\neq l}|\mathrm{Tr}(\mathcal{E}_k)|,
    \end{split}
\end{equation}
where $$\mathcal{E}_k=W_{l-1}\cdots W_{k+1} E_k W_{k+1}^\top \cdots W_{l-1}^\top \Tilde{V}_l^\top f_\alpha'(S^2)\Tilde{V}_l W_l^\top \cdots W_L^\top W_L \cdots W_l \Tilde{V}_l^\top f_\alpha'(S^2)\Tilde{V}_l$$ for $k<l$ and $$\mathcal{E}_k=(W_l^\top W_l)^{l-1}\Tilde{V}_l^\top f_\alpha'(S^2)\Tilde{V}_l W_l^\top \cdots W_{k-1}^\top E_k W_{k-1} W_l\Tilde{V}_l^\top f_\alpha'(S^2)\Tilde{V}_l$$ for $k>l$.

\paragraph{Second term in \eqref{eq:expansion}.}
Denote $\mathcal{E}_k= \mathcal{C}_k \Tilde{V}_l^\top f_\alpha'(S^2)\Tilde{V}_l$.
We define an operator $\mathcal{S}(A)$ equals to the sum of all singular values of $A$. Then by Fact \ref{fact:eigen}, $\left|\mathrm{Tr}(\mathcal{E}_k)\right|\le \mathcal{S}(\mathcal{E}_k)$. Since $\|W_s\|_2\le\|W_s\|_F\le \sqrt{C}$,  $\|E_k\|_2\le k\varepsilon_1 C^{k-1}$ for $k<l$ and $\|E_l\|_2\le (L-k+1)\varepsilon_1 C^{L-k}$ by \ref{thm:balance}. Since $\|\Tilde{V}_l\top f_\alpha'(S^2)\|_2 \Tilde{V}_l\le \mathrm{Tr}(f_\alpha'(S^2))$, we have $\|\mathcal{C}_k\|_2\le \mathrm{Tr}(f_\alpha'(S^2)) C^{L-1} k \varepsilon_1 $ for $k<l$ and $\|\mathcal{C}_k\|_2\le \mathrm{Tr}(f_\alpha'(S^2)) C^{L-1} (L-k+1) \varepsilon_1$ for $k>l$. Thus, by Fact \ref{fact:prod},
\begin{equation}
    \begin{split}
        \sum_{k\neq l}|\mathrm{Tr}(\mathcal{E}_k)|&\le \mathcal{S}(f_\alpha'(S^2)) \sum_{k\neq l}\|\mathcal{C}_k\|_2\\
        &\le \mathcal{S}(f_\alpha'(S^2))\mathrm{Tr}(f_\alpha'(S^2)) C^{L-1} \varepsilon_1 \left(\sum_{k<l}k+\sum_{k>l} (L-k+1)\right)\\
        &\le C^{L-1}\frac{L^2}{2}\varepsilon_1\left(\sum_{i=1}^{r_l} f_\alpha'(s_i^2)\right)^2.
    \end{split}
\end{equation}
\paragraph{First term in \eqref{eq:expansion}.}
By Fact \ref{fact:eigen} and Fact \ref{fact:prod}, we have
$$
\left|\mathrm{Tr}\left(\Tilde{V}_l^\top (f_\alpha'(S^2))^2S^{2L}\Tilde{V}_l\right)\right|\le \mathcal{S}\left(\Tilde{V}_l^\top (f_\alpha'(S^2))^2S^{2L}\Tilde{V}_l\right)\le \sum_{i=1}^{r_l} (f_\alpha'(s_i^2))^2s_i^{2L}\le (\sum_{i=1}^{r_l} f_\alpha'(s_i^2)s_i^{L})^2.
$$

Then we have that 
\begin{equation*}
    \begin{split}
        &\quad\sqrt{\mathrm{Tr}\left(W_{l-1}\cdots W_1 W_1^\top \cdots W_{l-1}^\top \Tilde{V}_l f_\alpha'(S^2)\Tilde{V}_l W_l^\top \cdots W_L^\top W_L \cdots W_l \Tilde{V}_l f_\alpha'(S^2)\Tilde{V}_l\right)}\\
        &\le \sqrt{(\sum_{i=1}^{r_l} f_\alpha'(s_i^2)s_i^{L})^2+C^{L-1}\frac{L^2}{2}\varepsilon_1(\sum_{i=1}^{r_l} f_\alpha'(s_i^2))^2}\\
        &\le \sum_{i=1}^{r_l} f_\alpha'(s_i^2)s_i^{L}+C^{\frac{L-1}{2}}L\sqrt{\varepsilon_1} \sum_{i=1}^{r_l} f_\alpha'(s_i^2)
    \end{split}
\end{equation*}
We denote $\delta:=\sum_{i=1}^{r_l}\lambda f_\alpha'(s_i^2)s_i^2-\sqrt{2(C_1+C^L)}f_\alpha'(s_i^2)s_i^L$ and $\mathcal{E}:=\sqrt{2(C_1+C^L)}C^{\frac{L-1}{2}}L\sqrt{\varepsilon_1} \sum_{i=1}^{r_l} f_\alpha'(s_i^2)$. Then 
$$
F_\alpha\circ\sigma\left(W_l(t+1)^\top W_l(t+1)\right)\le \sum_{i=1}^nf_\alpha(s_i^2)-2\eta\delta+2\eta\mathcal{E}+\beta.
$$
When $\sum_{i=1}^{r_l}f_\alpha(s_i^2)>r+\varepsilon_2/2$, we have $-2\eta\delta+2\eta\mathcal{E}+\beta\le 0$ by Lemma \ref{lemma:g epsilon} and when $\sum_{i=1}^{r_l}f_\alpha(s_i^2)\le r+\varepsilon_2/2$, we have $-2\eta\delta+2\eta\mathcal{E}+\beta\le \varepsilon_2/2$ by Lemma \ref{lemma:l epsilon}. Then $F_\alpha\circ\sigma\left(W_l(t+1)^\top W_l(t+1)\right)\le r+\varepsilon_2$.\qed

\subsection{Bounds on Error Terms}
The first derivative of $f_\alpha$ is 
$$
f_\alpha'(x)=
\left\{ 
\begin{aligned}
&\frac{2}{\alpha}-\frac{2}{\alpha^2}x,\ &x\le \alpha,\\
&0,\ &x>\alpha.
\end{aligned}
\right.
$$
The second derivative of $f_\alpha$ is
$$
f_\alpha''(x)=
\left\{ 
\begin{aligned}
&-\frac{2}{\alpha^2},\ &x\le \alpha,\\
&0,\ &x>\alpha.
\end{aligned}
\right.
$$
Thus, we have $f_\alpha(x)\in [0,1]$, $f_\alpha'(x)\in [0,\frac{2}{\alpha}]$ and $f_\alpha''(x)\le 0$.

\begin{lemma}
\label{lemma:eta2}
With same conditions and notations in Theorem \ref{thm:rank k}, the $O(\eta^2)$ term 
$$
\beta\le2\eta^2(2(C_1+C^L)C^{L-1}+\lambda^2C)\sum_{i=1}^{r_l} f_\alpha'(s_i^2).
$$
\end{lemma}

\noindent\textit{Proof.} As defined in Theorem \ref{thm:rank k}, 
\begin{equation}
    \begin{split}
    \label{eq:beta}
        \beta&=\left\langle\nabla(F_\alpha\circ\sigma)(W_l^\top W_l),\eta^2 (T_l+\lambda W_l)^\top(T_l+\lambda W_l)\right\rangle\\
    &\quad\quad+\nabla^2(F_\alpha\circ\sigma)(W_l^\top W_l+\gamma\eta\Delta)[\eta\Delta,\eta\Delta].
    \end{split}
\end{equation}
We bound the two terms in \eqref{eq:beta} separately.
\paragraph{First term in \eqref{eq:beta}.}
By Fact \ref{fact:prod} and the proof of Proposition \ref{prop:bound}, we have
\begin{equation}
    \begin{split}
        \left\langle\nabla(F_\alpha\circ\sigma)(W_l^\top W_l),\right.&\left.\eta^2 (T_l+\lambda W_l)^\top(T_l+\lambda W_l)\right\rangle\\
        \le&\eta^2\mathrm{Tr}(f_\alpha'(S^2))\|T_l+\lambda W_l\|_2^2\\
        \le&2\eta^2(2(C_1+C^L)C^{L-1}+\lambda^2C)\sum_{i=1}^{r_l}f_\alpha'(s_i^2)
    \end{split}
\end{equation}

\paragraph{Second term in \eqref{eq:beta}.}
By Lemma \ref{lemma:2-derivative}, 
$$
\begin{aligned}
    \nabla^2(F_\alpha\circ\sigma)&(W_l^\top W_l+\gamma\eta\Delta)[\eta\Delta,\eta\Delta]\\
    =&\eta^2\left[\nabla^2 F_\alpha(\sigma(W_l^\top W_l+\gamma\eta\Delta))[\mathrm{diag}\Tilde{\Delta},\mathrm{diag}\Tilde{\Delta}]+\langle \mathcal{A}(\sigma(W_l^\top W_l+\gamma\eta\Delta)),\Tilde{\Delta}\circ\Tilde{\Delta}\rangle\right]
\end{aligned}
$$
where $\Tilde{\Delta}=\Tilde{V}_l \Delta \Tilde{V}_l^\top$ and 
\begin{equation*}
\mathcal{A}_{ij}(\sigma(W_l^\top W_l+\gamma\eta\Delta))= \begin{cases}-\frac{2}{\alpha^2} & \text { if } i \neq j \text { but } \Tilde{s}_i^2, \Tilde{s}_j^2\le \alpha, \\ \frac{\frac{2}{\alpha}-\frac{2}{\alpha^2}\Tilde{s}_i^2}{\Tilde{s}_i^2-\Tilde{s}_j^2} & \text { if }\Tilde{s}_i^2\le \alpha,\ \Tilde{s}_j^2>\alpha,\\ \frac{\frac{2}{\alpha}-\frac{2}{\alpha^2}\Tilde{s}_j^2}{\Tilde{s}_j^2-\Tilde{s}_i^2} & \text { if }\Tilde{s}_i^2> \alpha,\ \Tilde{s}_j^2\le\alpha,\\ 0 & \text { otherwise, } \\ \end{cases}
\end{equation*}
where $\Tilde{s}_1,\dots \Tilde{s}_{r_l}$ are the eigenvalues of $W_l^\top W_l+\gamma\eta\Delta$. 
Since $\nabla^2 F_\alpha(\sigma(W_l^\top W_l+\theta\eta\Delta))=\mathrm{diag}\{f_\alpha''(\Tilde{s}_1^2),\dots,f_\alpha''(\Tilde{s}_{r_l}^2)\}$ with all entries non-positive, we have
$$\nabla^2 F_\alpha(\sigma(W_l^\top W_l+\gamma\eta\Delta))[\mathrm{diag}\Tilde{\Delta},\mathrm{diag}\Tilde{\Delta}]\le0.$$
Moreover, all entries of $\mathcal{A}(\sigma(W_l^\top W_l))$ are non-positive. Thus, 
$$\langle \mathcal{A}(\sigma(W_l^\top W_l+\gamma\eta\Delta)),\Tilde{\Delta}\circ\Tilde{\Delta}\rangle\le 0.$$

\noindent Overall, $\beta\le 2\eta^2(2(C_1+C^L)C^{L-1}+\lambda^2C)\sum_{i=1}^{r_l} f_\alpha'(s_i^2)$\qed

\begin{lemma}
    \label{lemma:g epsilon}
    With same conditions and notations in Theorem \ref{thm:rank k}, when $k+\varepsilon_2/2<\sum_{i=1}^{r_l} f_\alpha(s_i^2)\le r+\varepsilon_2$, we have $-2\eta\delta+2\eta\mathcal{E}+\beta\le 0.$
\end{lemma}

\noindent\textit{Proof.} Define $g(x)=f_\alpha'(x)\left(\lambda x-\sqrt{2(C_1+C^L)}x^{L/2}\right)$ on $x\ge 0$. Then since $L\ge 3$, for $\alpha\le\left(\frac{\lambda^2}{2(C_1+C^L)}\right)^\frac{1}{L-2}$, $f_\alpha'(x)=0$ when $x>\alpha$ and $\lambda x-\sqrt{2(C_1+C^L)}x^{L/2}\ge 0$ when $x\le \alpha$. Thus, $g(x)\ge 0$ for any $x\ge 0$. Since $\delta=\sum_{i=1}^{r_l} g(s_i^2)$, we have $\delta\ge 0.$

\noindent Note that there are at most $r$ $s_i$'s such that $f_\alpha(s_i^2)\ge \frac{2r+1}{2(r+1)}=:M_r$. Otherwise, $$\sum_{i=1}^{r_l} f_\alpha(s_i^2)\ge (r+1)\frac{2r+1}{2(r+1)}=r+1/2>r+\varepsilon_2.$$ Specifically, when $\sum_{i=1}^{r_l} f_\alpha(s_i^2)\ge r+\varepsilon_2/2$, we have 
\begin{equation}
\label{eq:small}
    \sum_{i:f_\alpha(s_i^2)<M_r}f_\alpha(s_i^2)\ge \sum_{i=1}^{r_l} f_\alpha(s_i^2)-r\ge \frac{\varepsilon_2}{2}.
\end{equation}

\noindent For $x$ such that $f_\alpha(x)=t<M_r$, we have $\frac{1}{\alpha^2} x^2 -\frac{2}{\alpha}x+t=0$ indicating $x=\alpha(1-\sqrt{1-t})\ge \frac{t}{2}\alpha.$ Then we have
\begin{equation}
    \begin{split}
        g(x)&=\frac{2}{\alpha}\left(\lambda x-\sqrt{2(C_1+C^L)}x^{L/2}\right)\left(1-\frac{1}{\alpha}x\right)\\
        &=\frac{2}{\alpha}x\lambda\left(1-\alpha^{-\frac{L-2}{2}}x^{\frac{L-2}{2}}\right)\left(1-\frac{1}{\alpha}x\right)\\
        &\ge \lambda t\sqrt{1-t}\left(1-(1-\sqrt{1-t})^{\frac{L-2}{2}}\right).
    \end{split}
\end{equation}
Since $L\ge 3$,
$$1-(1-\sqrt{1-t})^{\frac{L-2}{2}}\ge 1-\sqrt{1-\sqrt{1-t}}=\frac{\sqrt{1-t}}{1+\sqrt{1-\sqrt{1-t}}}\ge\frac{\sqrt{1-t}}{2}.$$
Then $g(x)\ge \frac{1}{2}\lambda t(1-t)>\frac{\lambda t}{4(r+1)}$. Thus, we have
\begin{equation}
    \delta=\sum_{i=1}^{r_l} g(s_i^2)\ge \sum_{i:f_\alpha(s_i^2)<M_r} g(sx-i^2)>\sum_{i:f(s_i^2)<M_r} f_\alpha(s_i^2) \frac{\lambda}{4(r+1)}\ge \frac{\lambda\varepsilon_2}{8(r+1)},
\end{equation}
where the last inequality is by \eqref{eq:small}. Note that $\sum_{i=1}^{r_l} f_\alpha'(s_i)\le \frac{2r_l}{\alpha}\le \frac{2n}{\alpha}$. Then since $\sqrt{\varepsilon_1}\le \frac{\lambda\alpha\varepsilon_2}{32nL(r+1)C^{\frac{L-1}{2}}\sqrt{2(C_1+C^L)}}$, we have $\mathcal{E}\le \frac{\lambda\varepsilon_2}{16(r+1)}$ and since $\eta\le \frac{\lambda\alpha\varepsilon_2}{32n(r+1)(2(C_1+C^L)C^{L-1}+\lambda^2 C)}$, we have $\beta\le \eta \frac{\lambda\varepsilon_2}{8(r+1)}$. Thus,
\begin{equation}
    -2\eta\delta+2\eta\mathcal{E}+\beta\le \eta\left(-\frac{\lambda\varepsilon_2}{4(r+1)}+\frac{\lambda\varepsilon_2}{8(r+1)}+\frac{\lambda\varepsilon_2}{8(r+1)}\right)=0.
\end{equation}\qed

\begin{lemma}
    \label{lemma:l epsilon}
    With same conditions and notations in Theorem \ref{thm:rank k}, when $\sum_{i=1}^{r_l} f_\alpha(s_i^2)\le r+\varepsilon_2/2$, we have $-2\eta\delta+2\eta\mathcal{E}+\beta\le \varepsilon_2/2.$
\end{lemma}

\noindent\textit{Proof.} Note that $\delta\ge 0$, $\mathcal{E}\le\frac{\lambda\varepsilon_2}{16(r+1)}$ and $\beta\le \eta\frac{\lambda\varepsilon_2}{8(r+1)}$. Since $\eta\le \frac{2(r+1)}{\lambda}$, we have
\begin{equation}
    -2\eta\delta+2\eta\mathcal{E}+\beta\le \eta\left(\frac{\lambda\varepsilon_2}{8(r+1)}+\frac{\lambda\varepsilon_2}{8(r+1)}\right)\le \frac{\varepsilon_2}{2}.
\end{equation}\qed
\section{Proof of Theorem \ref{thm:travel}} \label{sec:Proof_of_thm_5.2}
In Theorem \ref{thm:travel}, $T=T_0+T_1$. The following statements explain the change of $\theta_t$ during first $T_0$ iterations and last $T_1$ iterations respectively.
\begin{theorem}
    For any initialization $\theta_0$, denote $C_0:=\max_{1\le l\le L} \|W_l\|_F^2$, if $\eta\le\min\left\{\frac{C_1}{4(2(C_1+C_0^L)C_0^{L-1}+\lambda^2 C_0)},\frac{\lambda\varepsilon_1}{4(C_1+C^L)C^{L-1}+2\lambda^2 C}\right\}$ and $C\ge\frac{C_1}{\lambda},$ then for any time $T\ge \frac{\log(2C_0/\varepsilon_1)}{\eta\lambda}$ we have
    $$\theta_T\in B_{C,\varepsilon_1}$$
\end{theorem}

\textit{Proof.} Similar to the proof of \ref{prop:bound}, if $\eta\le \frac{C_1}{4(2(C_1+C_0^L)C_0^{L-1}+\lambda^2 C_0)}$ and $C\ge\frac{C_1}{\lambda}$, we have
$$\|W_l(t+1)\|_F^2\le (1-2\eta\lambda)\|W_l(t)\|_F^2+\eta C_1.$$
If $\|W_l(t)\|_F^2\ge C$, then $\|W_l(t+1)\|_F^2\le (1-\eta\lambda) \|W_L(t)\|_F^2$. Otherwise, $\|W_l(t+1)\|_F^2\le C$.
Thus, there exists $t\le T_0$ such that $\|W_l\|_F^2\le C$ for any $l$ when $T_0\ge\frac{\log(C_0/C)}{\eta\lambda}\ge \log\left(\frac{C}{C_0}\right)/\log\left(1-\eta\lambda\right)$.

After all weights satisfy $\|W_l\|_F^2\le C$, $\|W_lW_l^\top- W_{l+1}^\top W_{l+1}\|_2\le 2C$. Similar to the proof of \ref{prop:balance}, we have
\begin{equation}
    \begin{split}
        \|W_l(t+1)&W_l(t+1)^\top-W_{l+1}(t+1)^\top W_{l+1}(t+1)\|_2\\
        &\le (1-\eta\lambda)^2\|W_lW_l^\top- W_{l+1}^\top W_{l+1}\|_2+4\eta^2(C_1+C^L)C^{L-1}\\
        &\le (1-2\eta\lambda)\|W_lW_l^\top- W_{l+1}^\top W_{l+1}\|_2+2\eta^2\lambda^2 C+4\eta^2(C_1+C^L)C^{L-1}
    \end{split}
\end{equation}
When $\eta\le \frac{\lambda\varepsilon_1}{4(C_1+C^L)C^{L-1}+2\lambda^2 C}$, we have $\|W_l(t+1)W_l(t+1)^\top-W_{l+1}(t+1)^\top W_{l+1}(t+1)\|_2\le (1-\eta\lambda)\max\{\|W_lW_l^\top- W_{l+1}^\top W_{l+1}\|_2,\varepsilon_1\}$ for any $l$. Then for $T_1\ge \frac{\log(2C/\varepsilon_1)}{\eta\lambda}\ge\log\left(\frac{\varepsilon_1}{2C}\right)/\log\left(1-\eta\lambda\right)$, $\theta_{T_0+T_1}\in B_{C,\varepsilon_1}.$

\begin{theorem}
    For any parameter $\theta_t\in B_{\varepsilon_1,C}$ satisfying $\varepsilon_1\le \frac{\alpha\varepsilon_2}{4(n-r)(L-1)}$, then for any $T\ge \frac{\log\left((4(n-r)C)/(\alpha\varepsilon_2)\right)}{2\eta\lambda}$ we have
    $$\mathbb{P}(\theta_{t+T}\in B_{r,\varepsilon_1,\varepsilon_2,C}|\theta_t\in B_{\varepsilon_1,C})\ge \left(\frac{r}{\min\{d_{in},d_{out}\}}\right)^T.$$
\end{theorem}

\noindent\textit{Proof.} For the true matrix $A^*$, the number of columns is $d_{in}$ and the number of rows is $d_{out}$. Let $n=\min\{d_{in},d_{out}\}$. Without loss of generality we can assume that $n=d_{in}$, i.e. there are $n$ columns. We consider the $r$ columns with most observed entries and denote the the set of these entries by $J$. Then $|J|\ge \frac{r}{n}|I|$ and for each step $s$, the probability of sampling from $J$ is $\mathbb{P}((i_s,j_s)\in J)\ge \frac{r}{n}$. Then the event that all steps $s$ from $t$ to $t+T-1$, random entries $(i_s,j_s)$ are all sampled from $J$ has probability at least $\left(\frac{r}{n}\right)^T$. Under this event, we consider the weight of first layer $W_1$. We have
$$W_1(s+1)=(1-\eta\lambda)W_1(s)-\eta W_2(s)^\top \cdots W_L(s)^\top G_{\theta_s,i_sj_s}.$$
Then 
$$W_1(T)=(1-\eta\lambda)^T W_1(t)+\sum_{s=1}^T(1-\eta\lambda)^{T-s}W_2(t+s)^\top \cdots W_L(t+s)^\top G_{\theta_{t+s},i_{t+s}j_{t+s}}.$$
Since $(i_{t+s},j_{t+s})\in J$, the non-zero entry of $G_{\theta_{t+s},i_{t+s}j_{t+s}}$ is located on the $r$ columns supporting $J$, for any $s=1,\dots, T$. Thus, $W_2(t+s)^\top \cdots W_L(t+s)^\top G_{\theta_{t+s},i_{t+s}j_{t+s}}$ only has non-zero entries on those $r$ columns. Then the $r+1$'s singular value of $W_1(t+T)$ satisfies $\sigma_i(W_1(t+T))\le (1-\eta\lambda)^T \sqrt{C}$ for any $i>r$.

For $l>1$, we have $\|W_{l-1}W_{l-1}^\top-W_l^\top W_l\|_2\le \varepsilon_1$. Then $|\sigma_i(W_l^\top W_l)- \sigma_i(W_{l-1}W_{l-1}^\top)|\le\varepsilon_1$ for any $i$, i.e. $|\sigma_i(W_l)^2-\sigma_i(W_{l-1})^2|$ Then for any $l$, we have $\sigma_i(W_l(T))^2\le (1-\eta\lambda)^{2T} C+(l-1)\varepsilon_1.$ 
When $\varepsilon_1\le \frac{\alpha\varepsilon_2}{4(n-r)(L-1)}$ and $T\ge \frac{\log\left((4(n-r)C)/(\alpha\varepsilon_2)\right)}{2\eta\lambda}\ge \log\left(\frac{\alpha\varepsilon_2}{4(n-r)C}\right)/2\log(1-\eta\lambda)$, we have $\sigma_i(W_l(T))^2\le \frac{\alpha\varepsilon_2}{2(n-r)}$ for any $i>r$. Then for any $i>r$, $$f_\alpha(\sigma_i(W_l(T)^2))\le \frac{1}{\alpha^2}\frac{\alpha\varepsilon_2}{2(n-r)}\left(2\alpha-\frac{\alpha\varepsilon_2}{2(n-r)}\right)\le \frac{2}{\alpha}\frac{\alpha\varepsilon_2}{2(n-r)}\le \frac{\varepsilon_2}{n-r}$$ Thus, 
$$F_\alpha\circ \sigma(W_l(T)^\top W_l(T))\le r+\sum_{i=r+1}^n f_\alpha(\sigma_i(W_l(T)^2))\le r+\varepsilon_2.$$

If $n=d_{out}$, the proof is the same by selecting $r$ rows with most observed entries.
\section{Low Rank Property of $A_\theta$} \label{sec:low_rank_outputs}
In Proposition 4, we show that for any minimizer $\hat{\theta}$ in $B_{r,\varepsilon_1,\varepsilon_2,C}$, it is approximate rank-$r$. In fact, any general parameter $\theta\in B_{r,\varepsilon_1,\varepsilon_2,C}$ is approximate rank-$r$ or less:
\begin{proposition}
    \label{prop:theta}
    For any parameter $\theta\in B_{r,\varepsilon_1,\varepsilon_2,C}$, we have
    $$\sum_{i=1}^{\mathrm{Rank}\ A_\theta}f_\alpha(s_i(A_\theta^\top A_\theta))\le r+\varepsilon_2+\frac{L^2}{\alpha}C^{L-1}\varepsilon_1.$$
    Moreover, if $\varepsilon_1\le \frac{\alpha\varepsilon_2}{nL^2C^{L-1}}$, we have
    $$\sum_{i=1}^{\mathrm{Rank}\ A_\theta}f_\alpha(s_i(A_\theta^\top A_\theta))\le r+2\varepsilon_2.$$
\end{proposition}
\textit{Proof.} Since $(W_kW_k^\top)^k=(W_{k+1}^\top W_{k+1})^k+E_k$ and $E_k=\sum_{i=1}^k (W_{k+1}^\top W_{k+1})^{i-1} (W_kW_k^\top-W_{k+1}^\top W_{k+1})(W_kW_k^\top)^{k-i}$, we have
\begin{equation*}
    \begin{split}
        A_\theta A_\theta^\top&=  W_L \cdots W_1W_1^\top\cdots W_L^\top=(W_L W_L^\top)^L +\sum_{k=1}^{L-1} \mathcal{E}_k,
    \end{split}
\end{equation*}
where $\mathcal{E}_k=W_L\cdots W_{k+1} E_k W_{k+1}^\top \cdots W_L^\top$.
Then by Taylor's expansion, we have
\begin{equation}
    \begin{split}
    \label{eq:1}
        \sum_{i=1}^{r_L} &f_\alpha(s_i(A_\theta^\top A_\theta))=F_\alpha\circ\sigma(A_\theta A_\theta^\top)\\
        &=F_\alpha\circ\sigma(((W_LW_L^\top)^L+\mathcal{E}))\\
        &=F_\alpha\circ \sigma((W_LW_L^\top)^L)+\left\langle\nabla (F_\alpha\circ\sigma)((W_LW_L^\top)^L), \mathcal{E}\right\rangle\\
        &\quad\quad+\nabla^2 (F_\alpha\circ\sigma)((W_LW_L^\top)^L+\gamma\mathcal{E})[\mathcal{E},\mathcal{E}],
    \end{split}
\end{equation}
where $\gamma\in(0,1)$ and $\mathcal{E}=\sum_{k=1}^{L-1}\mathcal{E}_k$. Note that the $\mathrm{Rank}\ A_\theta\le r_L$, so we can let $s_i(A_\theta^\top A_\theta)=0$ for $i>\mathrm{Rank}\ A_\theta$.

\paragraph{First term in \eqref{eq:1}.} $F_\alpha\circ \sigma((W_LW_L^\top)^L)=\sum_{i=1}^{r_L} f_\alpha(s_i^{2L})$, where $\{s_1,\dots,s_{r_L}\}$ are the singular values of $W_l$. Then since $f_\alpha$ is non-decreasing, $f_\alpha(s_i^{2L})\le f_\alpha(s_i^2)$ for $s_i\le 1$ and $f_\alpha(s_i^{2L})=f_\alpha(s_i^2)=1$  for $s_i>1>\alpha$. Thus, $F_\alpha\circ\sigma((W_LW_L^\top)^L)\le F_\alpha\circ \sigma(W_L^\top W_L)\le r+\varepsilon_2.$

\paragraph{Second term in \eqref{eq:1}.} By Lemma \ref{lemma:1-derivative}, $\nabla (F_\alpha\circ\sigma)((W_LW_L^\top)^L)=\Tilde{U}_L \mathrm{diag}\{f_\alpha'(s_i^{2L}),\dots ,f_\alpha'(s_{r_l}^{2L})\}.$ Then by Fact \ref{fact:prod} and \ref{fact:eigen}, we have
\begin{equation*}
    \begin{split}
        \left\langle\nabla (F_\alpha\circ\sigma)((W_LW_L^\top)^L), \mathcal{E}\right\rangle&\le \mathcal{S}\left(\nabla (F_\alpha\circ\sigma)((W_LW_L^\top)^L)\right) \left\|\mathcal{E}\right\|_2\\
        &\le \|\mathcal{E}\|_2 \sum_{i=1}^{r_L} f_\alpha'(s_i^{2L})\\
        &\le \frac{2r_L}{\alpha}\|\mathcal{E}\|_2 \le \frac{2n}{\alpha}\|\mathcal{E}\|_2
    \end{split}
\end{equation*}
Since $\|E_k\|_F\le kC^{k-1} \varepsilon_1$, we have $\|\mathcal{E}_k\|_F\le kC^{L-1}\varepsilon_1$. Then $$\|\mathcal{E}\|_2\le \|\mathcal{E}\|_F\le \sum_{k=1}^{L-1}\|\mathcal{E}_k\|_F\le C^{L-1}\varepsilon_1\sum_{k=1}^{L-1}k\le \frac{L^2}{2} C^{L-1}\varepsilon_1.$$
Thus, $\left\langle\nabla (F_\alpha\circ\sigma)((W_LW_L^\top)^L), \mathcal{E}\right\rangle\le \frac{nL^2}{\alpha}C^{L-1}\varepsilon_1.$

\paragraph{Third term in \eqref{eq:1}.} By Lemma \ref{lemma:2-derivative}, 
\begin{equation*}
    \begin{split}
        \nabla^2 (F_\alpha\circ&\sigma)((W_LW_L^\top)^L+\gamma\mathcal{E})[\mathcal{E},\mathcal{E}]\\  &=\nabla^2F_\alpha\left(\sigma((W_LW_L^\top)^L+\gamma\mathcal{E})\right)[\mathrm{diag}\Tilde{\mathcal{E}},\mathrm{diag}\Tilde{\mathcal{E}}]+\langle\mathcal{A}(\sigma((W_LW_L^\top)^L+\gamma\mathcal{E})),\Tilde{\mathcal{E}}\circ\Tilde{\mathcal{E}}\rangle,
    \end{split}
\end{equation*}
where $\Tilde{\mathcal{E}}=\Tilde{U}_L \mathcal{E} \Tilde{U}_L^\top.$ Since the entries of $\nabla^2F_\alpha\left(\sigma((W_LW_L^\top)^L+\gamma\mathcal{E})\right)$ and $\mathcal{A}(\sigma((W_LW_L^\top)^L+\gamma\mathcal{E}))$ are all non-positive (follows the proof in \ref{lemma:eta2}), we have 
$$\nabla^2 (F_\alpha\circ\sigma)((W_LW_L^\top)^L+\gamma\mathcal{E})[\mathcal{E},\mathcal{E}]\le 0.$$

Therefore, we can add the three terms up and have $$\sum_{i=1}^{\mathrm{Rank}\ A_\theta}f_\alpha(s_i(A_\theta^\top A_\theta))\le r+\varepsilon_2+\frac{nL^2}{\alpha}C^{L-1}\varepsilon_1.$$



\end{document}